\documentclass[11pt]{article}

\usepackage[utf8]{inputenc}   
\usepackage[T1]{fontenc}
\usepackage{lmodern}          
\usepackage{geometry}         
\usepackage{microtype}        
\usepackage{aliascnt}

\usepackage{mathtools, amssymb, amsthm, bbm}
\usepackage{enumitem}

\usepackage[ruled,vlined]{algorithm2e}
\usepackage[noend]{algpseudocode}

\usepackage{graphicx}
\usepackage{booktabs}
\usepackage{caption}
\usepackage{subcaption}

\usepackage{threeparttable}

\usepackage{xcolor}

\usepackage{xcolor}
\definecolor{utburnt}{HTML}{BF5700}
\definecolor{charcoal}{HTML}{333333}
\definecolor{goldenrod}{HTML}{D99E00}
\definecolor{tealblue}{HTML}{007C91}

\usepackage[backref=page,bookmarks]{hyperref}
\hypersetup{
  colorlinks=true,
  linkcolor=utburnt,
  citecolor=tealblue,
  urlcolor=tealblue
}

\usepackage[nameinlink,capitalize]{cleveref}
\usepackage[numbers,sort&compress]{natbib}

\theoremstyle{plain}

\newtheorem{theorem}{Theorem}[section]

\newaliascnt{lemma}{theorem}
\newtheorem{lemma}[lemma]{Lemma}
\aliascntresetthe{lemma}
\crefname{lemma}{lemma}{lemmas}
\Crefname{lemma}{Lemma}{Lemmas}

\newaliascnt{corollary}{theorem}

\aliascntresetthe{corollary}
\crefname{corollary}{corollary}{corollaries}
\Crefname{corollary}{Corollary}{Corollaries}

\newaliascnt{proposition}{theorem}
\newtheorem{proposition}[proposition]{Proposition}
\aliascntresetthe{proposition}
\crefname{proposition}{proposition}{propositions}
\Crefname{proposition}{Proposition}{Propositions}

\newaliascnt{fact}{theorem}
\newtheorem{fact}[fact]{Fact}
\aliascntresetthe{fact}
\crefname{fact}{fact}{facts}
\Crefname{fact}{Fact}{Facts}

\newaliascnt{conjecture}{theorem}

\aliascntresetthe{conjecture}
\crefname{conjecture}{conjecture}{conjectures}
\Crefname{conjecture}{Conjecture}{Conjectures}

\newaliascnt{assumption}{theorem}
\newtheorem{assumption}[assumption]{Assumption}
\aliascntresetthe{assumption}
\crefname{assumption}{assumption}{assumptions}
\Crefname{assumption}{Assumption}{Assumptions}

\newaliascnt{definition}{theorem}
\newtheorem{definition}[definition]{Definition}
\aliascntresetthe{definition}
\crefname{definition}{definition}{definitions}
\Crefname{definition}{Definition}{Definitions}

\newaliascnt{remark}{theorem}
\newtheorem{remark}[remark]{Remark}
\aliascntresetthe{remark}
\crefname{remark}{remark}{definitions}
\Crefname{remark}{Remark}{Definitions}

\numberwithin{equation}{section}

\def\A{\mathcal{A}}

\def\F{\mathcal{C}}
\def\D{\mathcal{D}}

\def\F{\mathcal{F}}

\def\H{\mathcal{H}}
\def\I{\mathcal{I}}

\def\L{\mathcal{L}}
\def\X{\mathcal{X}}

\def\S{\mathbb{S}}

\newcommand*{\N}{{\mathbb{N}}}

\newcommand*{\R}{{\mathbb{R}}}

\let\eps\epsilon
\let\phi\varphi

\DeclareMathOperator*{\pr}{\mathbb{P}}

\DeclareMathOperator*{\E}{\mathbb{E}}

\DeclareMathOperator*{\argmax}{arg\,max}

\DeclareMathOperator{\poly}{poly}

\DeclareMathOperator{\sign}{sign}

\DeclareMathOperator{\ind}{\mathbbm{1}}

\newcommand{\opt}{\mathsf{opt}}

\newcommand{\cube}[1]{\{\pm 1\}^{#1}}

\newcommand{\ignore}[1]{} 

\newcommand*{\w}{\mathbf{w}}

\newcommand*{\x}{\mathbf{x}}
\newcommand*{\z}{\mathbf{z}}

\newcommand*{\Dtarget}{{\D^*}}

\newcommand{\concept}{f}

\newcommand{\pup}{p_{\mathrm{up}}}
\newcommand{\pdown}{p_{\mathrm{down}}}

\newcommand{\nats}{\mathbb{N}}

\newcommand{\Gauss}{\mathcal{N}}

\newcommand{\slackamp}{\sigma}

\newcommand{\Dlabeled}{\bar{\D}}

\newcommand{\Sinplabeled}{\bar{S}_{\mathrm{inp}}}

\newcommand{\Sclnlabeled}{\bar{S}_{\mathrm{cln}}}

\newcommand{\opttotal}{\opt_{\mathrm{total}}}

\renewcommand{\P}{\mathcal{P}}

\newcommand{\vt}{\mathbf{t}}

\newcommand{\mI}{\mathbf{I}}
\newcommand{\mW}{\mathbf{W}}

\newcommand{\fup}{f_{\mathrm{up}}}
\newcommand{\fdown}{f_{\mathrm{down}}}

\newcommand{\Sfarin}{S_{\mathrm{far}\text{-}\mathrm{in}}}
\newcommand{\Sfarout}{S_{\mathrm{far}\text{-}\mathrm{out}}}
\newcommand{\Sin}{S_{\mathrm{in}}}
\newcommand{\Sout}{S_{\mathrm{out}}}

\newcommand{\gup}{g_{\mathrm{up}}}
\newcommand{\gdown}{g_{\mathrm{down}}}

\newcommand{\gsa}{\mathsf{GSA}}

\newcommand{\mm}{\mathsf{MM}}

\newcommand{\coef}{\mathsf{coef}}

\DeclareMathOperator{\dist}{\mathsf{dist}}

\definecolor{burntorange}{rgb}{0.8, 0.33, 0.0}

\title{Sandwiching Polynomials for Geometric Concepts\\ with Low Intrinsic Dimension}
\author{%
    \begin{tabular}{ccc}
        \begin{tabular}{c}
            Adam R. Klivans\thanks{Supported by NSF award AF-1909204 and the NSF AI Institute for Foundations of Machine Learning (IFML).}\\ \texttt{klivans@utexas.edu} \\ UT Austin 
        \end{tabular} & 
        \begin{tabular}{c}
             Konstantinos Stavropoulos\thanks{Supported by the NSF AI Institute for Foundations of Machine Learning (IFML), and by the Apple Scholars in AI/ML PhD fellowship.} \\ \texttt{kstavrop@utexas.edu} \\ UT Austin
        \end{tabular}
        & 
         \begin{tabular}{c}
              Arsen Vasilyan\thanks{Supported by the NSF AI Institute for Foundations of Machine Learning (IFML).} \\ \texttt{arsenvasilyan@gmail.com} \\ UT Austin
         \end{tabular}
    \end{tabular}
}
\date{\today}

\begin{document}

\maketitle

\begin{abstract}%
Recent work has shown the surprising power of low-degree {\em sandwiching} polynomial approximators in the context of challenging learning settings such as learning with distribution shift, testable learning, and learning with contamination.  A pair of sandwiching polynomials approximate a target function in expectation while also providing \emph{pointwise} upper and lower bounds on the function's values.  In this paper, we give a new method
for constructing low-degree sandwiching polynomials that yield greatly improved degree bounds for several fundamental function classes and marginal distributions. In particular, we obtain degree $\mathrm{poly}(k)$ sandwiching polynomials for functions of $k$ halfspaces under the Gaussian distribution, improving exponentially over the prior $2^{O(k)}$ bound.  More broadly, our approach applies to function classes that are low-dimensional and have smooth boundary.

In contrast to prior work, our proof is relatively simple and directly uses the smoothness of the target function's boundary to construct sandwiching Lipschitz functions, which are amenable to results from high-dimensional approximation theory.  For low-dimensional polynomial threshold functions (PTFs) with respect to Gaussians, we obtain doubly exponential improvements without applying the FT-mollification method of Kane used in the best previous result. 
\end{abstract}

\section{Introduction}

Polynomial approximation has played a central role in computational learning theory for over three decades \citep{linial1993constant,kalai2008agnostically,klivans2008learning,chandrasekaran2024smoothed,pinto2025learning,koehler2025constructive}.  For example, the work of \cite{kalai2008agnostically} showed that if function class ${\cal C}$ admits low-degree polynomial approximators with respect to the input marginal $\D$, then ${\cal C}$ can be efficiently learned in the agnostic model of learning.  Moreover, subsequent results provide evidence that polynomial approximation is essentially \emph{necessary} for efficient agnostic learning with near-optimal error guarantees \citep{dachman2014approximate,diakonikolas2021optimality}.  In this sense, understanding whether a function class admits a good low-degree polynomial approximator is the key step in developing a provably efficient agnostic learner. 

Sandwiching polynomials impose a stronger, more structured form of approximation.  Fix a distribution ${\cal D}$.  A \emph{sandwiching pair} consists of two polynomials $\pdown$ and $\pup$ such that (i) they approximate $f$ on average over $D$, and (ii) they \emph{pointwise bound} the function for every input $\x$, namely
\[
\pdown(\x)\ \le\ f(\x)\ \le\ \pup(\x)\qquad\text{for all }\x.
\]
Thus, sandwiching polynomials are a special type of approximating polynomials: beyond small average error, the polynomial approximators are required to never ``cross'' the target function.  Recent work has shown that the existence of low-degree sandwiching polynomials is quite powerful and yields efficient algorithms for several notoriously challenging learning tasks including testable learning, learning under distribution shift, and learning with contamination (\cite{rubinfeld2022testing, gollakota2022moment, goel2024tolerant, chandrasekaran2024efficient,klivans2023testable, klivans2025the}).  

Despite their growing importance, the sandwiching degree of many natural function classes is still poorly understood. For example, for the class of functions of $k$ halfspaces with respect to the Gaussian distribution, the best previously known degree bound (prior to this work) was $2^{O(k)}$ \citep{gopalan2010fooling,gollakota2022moment}. 

Here we give a general methodology for constructing sandwiching polynomials, leading to substantially improved degree bounds in several settings.  As a concrete example, our approach yields degree-$\widetilde{O}(k^5)$ sandwiching polynomials for the class of functions of $k$ halfspaces under the Gaussian distribution, improving exponentially over the previous $2^{O(k)}$ bound of \cite{gopalan2010fooling,gollakota2022moment}.  We obtain a {\em doubly} exponential degree improvement for low-dimensional polynomial threshold functions (PTFs).   More generally, our approach applies to any function class that satisfies two broad conditions: (i) \emph{low-dimensionality}, meaning each function depends only on a projection onto a low-dimensional subspace (e.g., the span of $k$ halfspace normals has dimension at most $k$), and (ii) a \emph{$\sigma$-smooth boundary}, meaning that any $\rho$-neighborhood of the decision boundary has probability mass at most $\sigma\rho$.  

In addition, our approach is not limited to Gaussian marginals: we handle a wide range of distributions far beyond the Gaussian case, namely arbitrary strictly subexponential distributions.  See \Cref{table:comparison} for a summary of our sandwiching-degree bounds and the prior state of the art.  Our new degree bounds for the above function classes lead to state-of-the-art running times for corresponding algorithms in testable learning, learning with distribution shift, and learning with heavy contamination.

\subsection{Our Results}

Our main result concerns the sandwiching degree of concepts with low intrinsic dimension and smooth boundary with respect to a strictly subexponential distribution. The sandwiching degree of a concept $f$ relative to a $\D$ is formally defined as follows.
\begin{definition}[Sandwiching Degree]\label{definition:sandwiching}
    For a function $f:\R^d\to \cube{}$, a distribution $\D$ over $\R^d$, $\eps\in(0,1)$ and $s\ge 1$, we say that the $(\eps,s)$-sandwiching degree of $f$ with respect to $\D$ is $\ell$ if there are polynomials $\pup,\pdown$ of degree at most $\ell$ such that the following conditions hold:
    \begin{enumerate}
        \item $\pdown(\x)\le f(\x) \le \pup(\x)$ for all $\x\in\R^d$.
        \item $\|\pup-\pdown\|_{\D,s} := (\E_{\x\sim\D}[|\pup(\x)-\pdown(\x)|^s])^{1/s} \le \eps$.
    \end{enumerate}
    The sandwiching degree of a concept class $\F$ is defined as the supremum of the corresponding sandwiching degrees of its elements.
\end{definition}

\begin{theorem}[Main Theorem]\label{theorem:main-intro}
    The $(\eps,s)$-sandwiching degree of concepts with intrinsic dimesion $k$ and $\sigma$-smooth boundary with respect to a $\gamma$-strictly subexponential distribution $\D$ is:
\[
        \ell(\eps,s)
        \;\le\;
        \widetilde{O}\!\left(
        \frac{\sigma k^{3/2} s}{(\eps/2)^{s+1}}
        \right)^{1+1/\gamma}.
\]
\end{theorem}

\begin{table*}[ht]
\begin{center}
\setlength{\tabcolsep}{9pt}
\begin{threeparttable}
\begin{tabular}{c c c c} 
 \toprule
 \textbf{Concept Class} & \begin{tabular}{c}\textbf{This Work}\end{tabular} & \begin{tabular}{c}\textbf{Prior Work}\end{tabular}&\textbf{References} \\ \midrule
    \begin{tabular}{c} Intersections of $k$ Halfspaces \end{tabular} & \begin{tabular}{c} ${\widetilde{O}(k^{3})}$ \end{tabular} & ${{{{O}}(k^6)}}$ & \begin{tabular}{c}
      \cite{gopalan2010fooling} 
  \end{tabular}
 \\ \midrule
  \begin{tabular}{c} Functions of $k$ Halfspaces \end{tabular} & \begin{tabular}{c} ${\widetilde{O}(k^{5})}$ \end{tabular} & ${\exp{{{O}}(k)}}$ & \begin{tabular}{c}
      \cite{gopalan2010fooling}\\ 
      \cite{klivans2013moment}\\
      \cite{gollakota2022moment}
  \end{tabular}
 \\ \midrule
 \begin{tabular}{c} Convex Sets in $k$ dimensions  \end{tabular} &  \begin{tabular}{c}${\widetilde{O}(k^{5})}$ \end{tabular}& \begin{tabular}{c}None\tnote{$\ast$} \end{tabular} & \begin{tabular}{c}
     \cite{de2023gaussian} 
 \end{tabular}
 \\ \midrule
 \begin{tabular}{c} Degree-$q$ PTFs in $k$ dimensions \end{tabular} &  ${\widetilde{O}(q^6k^5)}$ & \begin{tabular}{c}${\exp \exp{{{{O}}(q)}}}$ \\ or worse\tnote{$\dagger$}\end{tabular} & \begin{tabular}{c} \cite{kane2011kindependent} \\ \cite{slot2024testably}\end{tabular}
  \\ \midrule
 \begin{tabular}{c} Functions of $t$ concepts each \\ with intrinsic dimension $k$ \\ and $\sigma$-smooth boundary\tnote{$\ddagger$} \end{tabular} &  ${\widetilde{O}(\sigma^2 t^5k^3)}$ & None & --
 \\ \bottomrule
\end{tabular}
\begin{tablenotes}
\footnotesize
\item[$\ast$] The work of \cite{de2023gaussian} together with results by \cite{gopalan2010fooling} imply the existence of upper sandwiching polynomials of degree ${\exp{{{\widetilde{O}}}(k)}}$ for $k$-dimensional convex sets that are truncated into a ball of radius $\poly(k)$. 
\item[$\dagger$] For PTFs, the dependence on the degree $q$ is not made explicit in prior work \citep{slot2024testably} using FT-mollification due to \cite{kane2011kindependent}, but is at least doubly exponential. Their bounds are independent of the intrinsic dimension $k$.
\item[$\ddagger$] These concepts may depend on different $k$-dimensional subspaces. Boundary smoothness is w.r.t. the Gaussian.
\end{tablenotes}
\vspace{-1em}
\end{threeparttable}
\end{center}
\caption{Comparison between upper bounds from this work and the best known bounds in previous work on the $(\eps = 0.1, s = O(1))$-sandwiching degree of various geometric concept classes with respect to the standard Gaussian distribution.}
\label{table:comparison}
\end{table*}

In \Cref{table:comparison}, we summarize the implications of our main result on the sandwiching degree of several geometric concept classes under the Gaussian distribution (see also \Cref{section:sandwiching-bounds-for-classes}). Each of the entries in the table corresponds to the state-of-the-art results before and after our work for each of the problems outlined in \Cref{intro:applications} (see also \Cref{section:applications}). These results are obtained by combining our main sandwiching degree bound with bounds for the boundary smoothness parameter of the corresponding classes (see \Cref{section:sandwiching-bounds-for-classes}).

Our result for functions of $k$ halfspaces also works beyond the Gaussian distribution, as long as the distribution $\D$ is strictly subexponential and anticoncentrated in every direction (see \Cref{theorem:sandwiching-functions-of-halfspaces}).

\subsection{Our Techniques}

\paragraph{Sandwiching Geometric Concepts by Lipschitz Functions}
Our first step (\Cref{lemma:sandwiching-property-smoothed-dilation}) establishes the existence of two Lipschitz functions $\fup,\fdown:\R^d\to[-1,1]$ satisfying
\begin{gather}
    \fdown(\x) \le f(\x) \le \fup(\x), \\
    \E_{\x\sim\D}[\fup(\x)-\fdown(\x)] \le \eps,
\end{gather}
where $f$ is a geometric concept and $\D$ is a strictly subexponential distribution.

We define one-sided relaxations $f^{+\rho}$ and $f^{-\rho}$, where $f^{+\rho}$ (resp.\ $f^{-\rho}$) equals $1$ on points within distance $\rho$ of the interior (resp.\ exterior) of $f$. We then construct $\fup$ as a $(1/\rho)$-Lipschitz interpolation between $f$ and $f^{+\rho}$; such an interpolation exists since their decision boundaries are separated by distance $\rho$. The function $\fdown$ is constructed analogously using $f^{-\rho}$. By construction, this immediately yields the pointwise sandwiching property.

To establish the approximation guarantee, we use the fact that if $f$ has $\sigma$-smooth boundary, then the $\rho$-neighborhood of its decision boundary has measure at most $\sigma\rho$ under $\D$. This bounds the expected gap between $f^{+\rho}$ and $f^{-\rho}$, and hence between $\fup$ and $\fdown$, with $\rho$ chosen appropriately.

\paragraph{Sandwiching Polynomials for Lipschitz Functions}
Next, we construct polynomial sandwiching approximations for $\fup$ and $\fdown$. We focus on $\fup$. By multivariate Jackson’s theorem \citep{Newman1964}, there exists a polynomial $p_1$ that uniformly approximates $\fup$ within a ball of radius $R$, leveraging the Lipschitzness of $\fup$. A result of \cite{ben2018classical} further controls the growth of $p_1$ outside this region, which is essential for bounding expectations under strictly subexponential distributions $\D$.

While these approximation tools are also used in \cite{chandrasekaran2025learning} for learning Lipschitz neural networks under distribution shift, their approach does not yield sandwiching polynomials. To obtain an explicit upper polynomial, we construct a polynomial $p_2$ that is at most $\eps$ within a ball of radius $R/2$ and dominates $p_1$ outside the ball of radius $R$ (see \Cref{fig:regions-sandwiching}). Setting $\pup = p_1 + p_2 + \eps$ then yields a valid upper sandwiching polynomial. A symmetric construction gives the lower polynomial.

\begin{figure}
    \centering
    \includegraphics[width=0.3\linewidth]{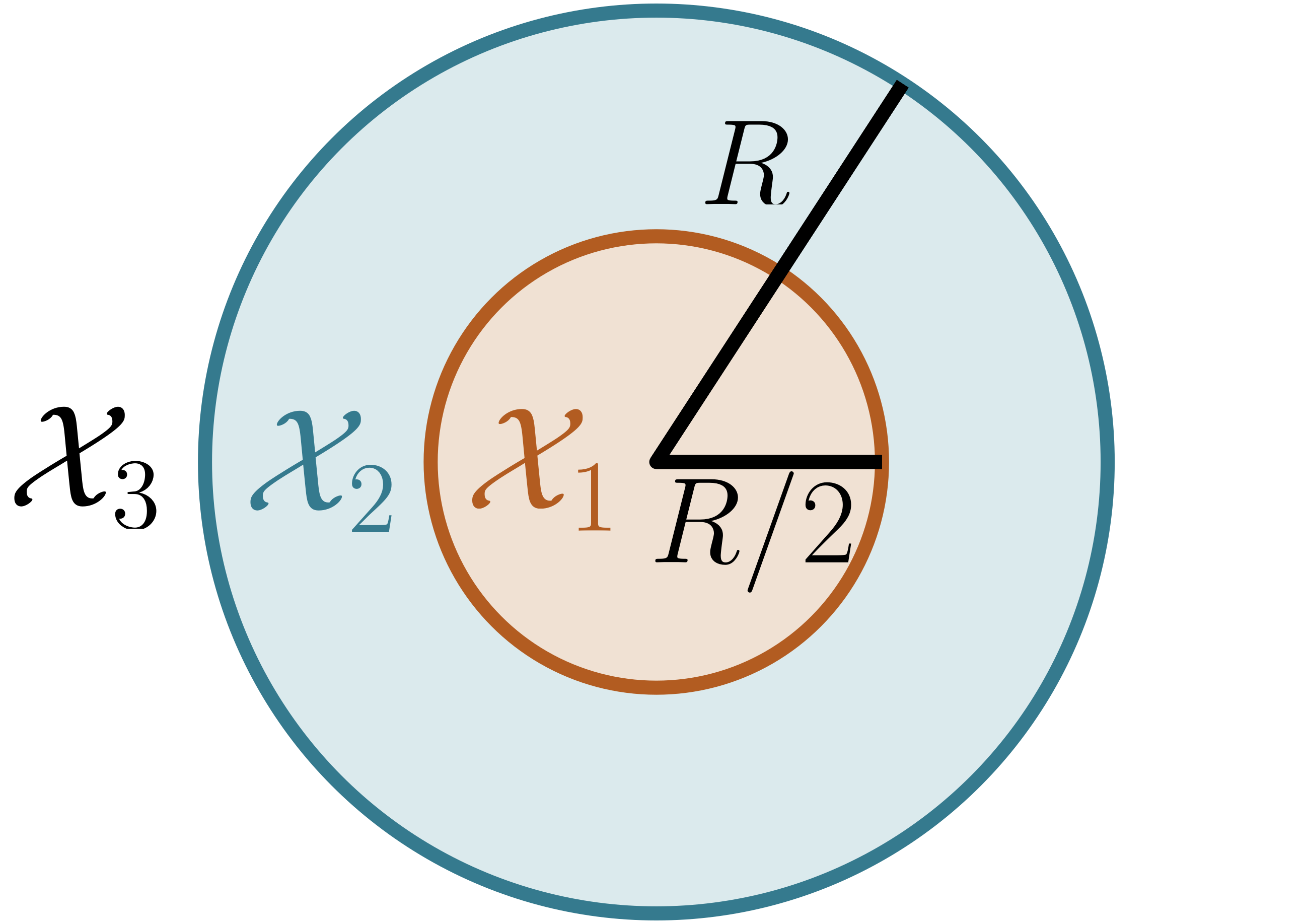}
    \caption{Our upper sandwiching polynomial for $\fup$ is of the form $\pup = p_1(\x)+p_2(\x)+\eps$. In region $\X_1$, $p_1$ is a pointwise $\eps$-approximator for $\fup$ and $0\le p_2(\x)\le \eps$, so $\fup(\x)\le \pup(\x)\le \fup(\x)+ 3\eps$. In region $\X_2$, $p_1$ is still an pointwise $\eps$-approximator for $\fup$ and $p_2(\x) > 0$, so $\fup(\x) \le p_1(\x)+\eps \le \pup(\x)$. In region $\X_3$ we have $p_2(\x) \ge 1+|p_1(\x)| \ge \fup(\x)+|p_1(\x)|$, so $\pup(\x) \ge \fup(\x)$.}
    \label{fig:regions-sandwiching}
\end{figure}

\paragraph{Comparison to \cite{gopalan2010fooling}}
The work of \cite{gopalan2010fooling} constructs sandwiching polynomials for arbitrary functions of $k$ halfspaces
\[
f(\x)=G(\sign(\w_1\cdot\x-\tau_1),\sign(\w_2\cdot\x-\tau_2),\dots,\sign(\w_k\cdot\x-\tau_k))
\]
under the Gaussian distribution. Their approach begins with the one-dimensional sandwiching polynomials of \cite{diakonikolas2010bounded}, which approximate the sign function pointwise within a bounded region around the origin, except near the discontinuity, and always lie above or below the sign function.

They lift these univariate polynomials to $\R^d$ by composing them with the linear forms $\w_i\cdot\x-\tau_i$, and then combine the resulting approximations via addition and multiplication according to the structure of $G$. The resulting sandwiching degree is $\poly(s)$, where $s$ is the size of $G$ when viewed as a decision tree. Since $s$ can be exponential in the number $k$ of inputs to $G$, this yields an $\exp(O(k))$ bound on the sandwiching degree.

In contrast, our approach is inherently high-dimensional and relies on tools from multivariate polynomial approximation theory, rather than composing one-dimensional sandwiching constructions. This allows us to obtain an exponential improvement in the sandwiching degree over \cite{gopalan2010fooling} for functions of halfspaces.

\subsection{Applications}\label{intro:applications}
In this section we describe the relationship between sandwiching polynomials and various learning frameworks.  Our improved bounds on sandwiching immediately imply corresponding improvements in these models. 

\paragraph{Tolerant Testable Learning}
In \emph{testable agnostic learning} (\Cref{definition:testable-learning}) \citep{rubinfeld2022testing}, the learner receives labeled samples from an arbitrary distribution over $\R^d \times \cube{}$ and is allowed to \emph{abstain}: it either accepts and outputs a hypothesis whose error is near-optimal within a target concept class, or rejects upon detecting a violation of a structural assumption on the $\R^d$-marginal. This rejection option allows testable learners to bypass computational hardness that arises in the absence of assumptions on the marginal, even for simple classes \citep{daniely2016complexity,diakonikolas2022near,diakonikolas2022cryptographic}.

A common strategy for obtaining efficient testable learning algorithms is to construct low $\L_1$-sandwiching degree \citep{gollakota2022moment}. However, the original version of testable learning can be fragile: the learner may reject even under small deviations from the target marginal $\D^*$. This motivates \emph{tolerant testable learning} (\Cref{definition:tolerant-testable-learning}) \citep{goel2024tolerant}, where the learner is guaranteed to accept with high probability whenever the marginal is sufficiently close to $\D^*$ in total variation distance, while retaining the same near-optimal guarantee upon acceptance. Prior work shows that $\L_1$-sandwiching degree suffices for efficient tolerant testable learning as well \citep{goel2024tolerant,klivans2025the}.

\paragraph{Learning with Distribution Shift}
In learning with distribution shift \citep{ben2006analysis,mansour2009domadapt,ben2010theory}, the learner receives labeled samples from a training distribution and only \emph{unlabeled} samples from a potentially different test distribution. The goal is to produce a hypothesis with low test error, or to reliably detect when the shift is too severe for accurate prediction.

A prominent formalization is \emph{testable learning with distribution shift} (TDS learning, \Cref{definition:tds-learning}) \citep{klivans2023testable}, which again equips the learner with a rejection option: it either accepts and outputs a hypothesis whose test error is near-optimal (see \Cref{definition:lambda}), or rejects upon detecting harmful shift. By combining our main theorem with the algorithmic framework of \cite{chandrasekaran2024efficient,klivans2023testable}, we obtain efficient TDS learners for the settings captured by \Cref{table:comparison}.

A stronger primitive is \emph{PQ learning} (\Cref{definition:pq-learning}) \citep{goldwasser2020beyond}, which requires \emph{per-point} abstention: the learner outputs a hypothesis together with a selector that can reject individual test points, while guaranteeing a low rejection rate on the training distribution. PQ learning implies (tolerant) TDS-style guarantees \citep{klivans2023testable,goel2024tolerant}, and the first end-to-end efficient PQ algorithms were obtained by extending TDS techniques using sandwiching polynomials \citep{goel2024tolerant,klivans2023testable}. Notably, all existing PQ-learning analyses rely on \emph{$\L_2$-sandwiching} with respect to the reference distribution, and it remains an open question whether analogous guarantees can be obtained assuming only \emph{$\L_1$-sandwiching} (in this work we do obtain $\L_s$-sandwiching for any $s$).

\paragraph{Learning with Heavy Contamination}
Sandwiching polynomials also yield efficient algorithms in the \emph{heavy contamination} model of \cite{klivans2025the} (\Cref{definition:heavy-contamination}), where only a constant fraction of the labeled dataset is drawn from a clean distribution and the remainder may be adversarially corrupted. The objective is to compete with the best classifier in the concept class on the \emph{entire} contaminated sample, namely to achieve error on the clean distribution proportional to the minimum empirical error on the corrupted dataset.  In this setting, \cite{klivans2025the} show that low-degree $\L_1$-sandwiching polynomials imply efficient learning, even though only a small fraction of the input data is structured.

\subsection{Related Work}

\paragraph{Pseudorandomness}
Sandwiching polynomials play a central role in pseudorandomness \citep{hatami2023theory}. The goal is to derandomize randomized algorithms by replacing a high-entropy \emph{base} distribution $\D$ (e.g., the uniform distribution on $\{\pm1\}^n$) with a \emph{pseudorandom} distribution $\D'$ that is generated from a short random seed, i.e., $\D'$ is the output distribution of a map $G:\{0,1\}^r\to \Omega$ for small $r$. We say that $\D'$ \emph{fools} a class $\F$ of test functions if it preserves their expectations, namely
\[
\bigl|\E_{\x\sim \D}[f(\x)]-\E_{\x\sim \D'}[f(\x)]\bigr|\le \eps
\quad\text{for all } f\in \F.
\]
The \emph{sandwiching degree} $\ell$ of $\F$ governs the seed length of a natural family of PRGs that fool $\F$ by ensuring that $\D'$ matches the moments of $\D$ up to degree $\ell$ (see \Cref{theorem:pseudorandomness-via-moment-matching}, due to \cite{bazzi2009polylogarithmic,gollakota2022moment}).

Most prior work bounds sandwiching degree primarily for $\D$ equal to the uniform distribution on the Boolean cube \citep{bazzi2009polylogarithmic,razborov2009simple,gopalan2010fooling,diakonikolas2010bounded,diakonikolas2010ptf,braverman2011poly,tal2017tight}, though there are also results for continuous distributions such as the Gaussian \citep{kane2011kindependent,gopalan2010fooling}. For each class in \Cref{table:comparison}, our bounds yield improved moment-matching PRGs (see \Cref{section:pseudorandomness-application}). While stronger pseudorandomness techniques can achieve even shorter seeds \citep{meka2010pseudorandom,o2022fooling}, they typically do not yield explicit sandwiching-degree guarantees and often define generators via conditions that appear difficult to verify. In contrast, moment matching is efficiently checkable, and sandwiching-degree bounds are exactly what our learning-theoretic applications require.

\paragraph{Learning via Polynomial Approximation}
Low-degree polynomial approximation has been a core tool in computational learning theory for over three decades \citep{linial1993constant,kalai2008agnostically,klivans2008learning,chandrasekaran2024smoothed,pinto2025learning,koehler2025constructive}. A classic result of \cite{linial1993constant} shows that $\L_2$ approximation---the existence of a low-degree polynomial $p$ with small squared error $\E[(p(\x)-f(\x))^2]$ under the input marginal---yields efficient learning.  Subsequent work \cite{klivans2008learning} showed that $\L_1$ approximation is sufficient even for agnostic learning. Moreover, later results provide evidence that $\L_1$ approximation is essentially \emph{necessary} if one wants efficient learning with near-optimal error \citep{dachman2014approximate,diakonikolas2021optimality}.

A notable line of work ties polynomial approximation under the Gaussian distribution to geometric measure theory. In particular, \cite{klivans2008learning} related Gaussian \emph{surface area} bounds for a concept class to the existence of low-degree approximating polynomials, and hence to efficient agnostic learning under the Gaussian. Their argument combines approximation properties of the Ornstein--Uhlenbeck semigroup \citep{pisier1986probabilistic,ledoux1994semigroup} with surface-area bounds for several geometric classes \citep{ball93,nazarov2003maximal}. Surface area can be viewed as a limiting ``boundary smoothness'' parameter: it measures the first-order rate at which a random point falls within distance $\rho$ of the boundary as $\rho\to 0$.

We can view our work as giving an analogue of the ``learning via Gaussian surface area'' paradigm but for \emph{reliable} learning primitives, where the relevant notion is not merely approximation in expectation but the stronger requirement of \emph{sandwiching}. Quantitatively, our bounds incur an explicit polynomial dependence on the intrinsic dimension parameter $k$ of the concept class. This dependence arises from our main approximation tool \citep{Newman1964}. For example, for intersections of $k$ halfspaces, \cite{klivans2008learning} obtains only logarithmic dependence on $k$, whereas our bounds scale as $\poly(k)$, but in exchange we obtain sandwiching guarantees (and, more generally, $\L_s$-sandwiching for arbitrary $s\ge 1$) rather than approximation in expectation.

\paragraph{Learning via Sandwiching}
Bounds on the sandwiching degree yield efficient algorithms for a range of recently defined learning primitives that impose strong reliability requirements \citep{goldwasser2020beyond,rubinfeld2022testing,gollakota2022moment,klivans2023testable,goel2024tolerant,chandrasekaran2024efficient,klivans2025the}. For many of these tasks, $\L_1$-sandwiching is sufficient. In contrast, \emph{current} algorithmic frameworks for PQ learning rely essentially on $\L_2$-sandwiching, and $\L_2$ bounds are also analytically convenient and may be useful beyond PQ learning. This is where the flexibility of our results matters: we obtain sandwiching guarantees in $\L_s$ for \emph{arbitrary} orders $s\ge 1$.

The first appearance of sandwiching polynomial approximation in the context of efficient learning algorithms is due to \cite{klivans2013moment}.  Their approach combines the duality between sandwiching and fooling via moment matching \citep{bazzi2009polylogarithmic} with the method of distances \citep{klebanov1996proximity,zolotarev1984probability,rachev2013methods}. Because their guarantee holds for general log-concave marginals, it does not require strict subexponential tails; however, the resulting degree bound is doubly exponential in $k$. By contrast, our bounds are $\poly(k)$.

More recently, \cite{gollakota2022moment} built on \cite{klivans2013moment} in the context of testable learning, obtaining sandwiching polynomials of degree $\exp(O(k))$ with bounded coefficients for functions of $k$ halfspaces under any strictly subexponential and anticoncentrated distribution. Our bounds improve exponentially over \cite{gollakota2022moment}.

\paragraph{Learning via Boundary Smoothness}
The notion of \emph{boundary smoothness} was introduced by \cite{chandrasekaran2024efficient} in the context of TDS learning. Using this condition, they gave a \emph{realizable} TDS learner for convex sets of intrinsic dimension $k$ under Gaussian training marginals, with running time $\poly(d)\cdot 2^{\poly(k/\eps)}$. Their guarantee, however, assumes that both the training and test labels are generated by the \emph{same} target convex set (i.e., realizability), and it additionally imposes a non-degeneracy condition requiring the positive region to have non-negligible Gaussian mass.

In contrast, our results yield a TDS learner running in time $d^{\poly(k/\eps)}$ that achieves \emph{near-optimal error guarantees} in the fully agnostic setting (see \Cref{definition:lambda}), which is the state of the art for this level of reliability.

At a high level, \cite{chandrasekaran2024efficient} combines a dimension-reduction step from classical learning theory \citep{vempala2010learning} with a distribution-shift tester: after recovering a $k$-dimensional subspace, the tester forms a grid in this subspace and verifies that the test distribution assigns approximately the same mass to each grid cell as the Gaussian training marginal.

\paragraph{Testable Learning under Relaxed Error Guarantees}
As discussed above, the best known upper bounds on the computational complexity of testable agnostic learning (\Cref{definition:testable-learning}) and TDS learning (\Cref{definition:tds-learning}) are obtained via sandwiching approximation. There are, however, faster algorithms for both frameworks under \emph{relaxed error guarantees} \citep{gollakota2023efficient,diakonikolas2023efficient,gollakota2024tester,diakonikolas2024testable,klivans2024learning,chandrasekaran2024efficient}.

For testable learning, these improved runtimes are currently known only for halfspaces. In the TDS setting, some results extend beyond halfspaces, but either incur substantially weaker error guarantees (for example, \cite{chandrasekaran2024efficient} obtain guarantees for intersections of $k$ halfspaces with error that degrades exponentially in $k$), or require realizability assumptions, namely that the labels for both the training and test distributions are generated by the same function in the target concept class.

\section{Preliminaries}

For a function $g:\R^d\to \R$ and a distribution $\D$ over $\R^d$, we define: 
$    \|g\|_{\D,s} := (\E_{\x\sim \D}[|g(\x)|^s])^{\frac{1}{s}} $.
For a set $K\subseteq\R^d$ and $\x\in\R^d$, we define $\dist(\x,K) = \inf_{\x'\in K}\|\x-\x'\|_2$. We denote with $[t]_a^b = \max\{a, \min\{t,b\}\}$ the $(a,b)$-clipping function. A \emph{concept} is a function $f:\R^d\to \cube{}$, while a \emph{concept class} is a set of concepts $\F\subseteq\{\R^d\to \cube{}\}$. We denote with $\Gauss(0,\mI_{d\times d})$ or simly $\Gauss_d$ the standard Gaussian distribution in $d$ dimensions.
For a vector $\x\in\R^d$, we denote with $x_i$ the $i$-th coordinate of $\x$. A polynomial $p$ over $\R^d$ of degree $\ell$ is a function of the form 
$p(\x)=\sum_{\I\in\N^d} c_p(\I)\, \x^\I$ where $x^\I = \prod_{i\in[d]}x_i^{\I_i}$ and $c_p(\I) = 0$ for any $\I$ with $\|\I\|_1 > \ell$. We denote with $\|p\|_\coef$ the quantity $\sum_{\I\in\N^d} |c_p(\I)|$, i.e., the sum of the absolute values of the coefficients of $p$.

We will use dilation and erosion operations on concepts; these correspond, respectively, to taking the Minkowski sum and Minkowski subtraction of the positive regions with a Euclidean ball of fixed radius as defined below.
\begin{definition}[Dilation and Erosion]
    Let $f:\R^d\to \cube{}$. For $\rho \ge 0$, we define the $\rho$-dilation $f^{+\rho}$, as well as the $\rho$-erosion $f^{-\rho}$ of $f$ as follows.
    \begin{align*}
        f^{+\rho}(\x) &= \sup_{\z: \|\z\|_2\le \rho}f(\x+\z)
        ,\quad f^{-\rho}(\x) = \inf_{\z: \|\z\|_2\le \rho}f(\x+\z)
    \end{align*}
\end{definition}

The boundary smoothness parameter relative to a distribution $\D$ is formally defined as follows using the dilation and erosion operations.
\begin{definition}[Smooth Boundary \cite{chandrasekaran2024efficient}]\label{definition:smooth-boundary}
    Let $f:\R^d \to \cube{}$ and $\D$ be a distribution over $\R^d$. For $\slackamp\ge 1$, we say that $f$ has $\slackamp$-smooth boundary with respect to $\D$ if for any $\rho \ge 0$
    \[
        \E_{\x\sim \D}\biggr[\frac{f^{+\rho}(\x) - f^{-\rho}(\x)}{2}\biggr] \le \slackamp\rho
    \]
\end{definition}

\begin{remark}
    Note that the quantity $\frac{1}{2}\E_{\x\sim \D}[f^{+\rho}(\x) - f^{-\rho}(\x)]$ equals the probability that a sample $\x$ from $\D$ is $\rho$-close to the boundary of $f$, i.e., there exists $\z: \|\z\|_2\le \rho$ such that $f(\x+\z) \neq f(\x)$.
\end{remark}

Our results require the following concentration assumption on the distribution $\D$.
\begin{definition}[Strictly Subexponential Distributions]\label{definition:subexponential}
For $\alpha,\beta,\gamma >0$, we say that a distribution $\D$ over $\R^d$ is $\gamma$-strictly subexponential with parameters $\alpha,\beta$ if for any $\w\in\S^{d-1}$ and any $r\ge 0$,
\[
    \pr_{\x\sim\D}[|\w\cdot \x| \ge r] \le \alpha\,\exp(-\beta \, r^{1+\gamma}).
\]
In what follows, we may omit the parameters $\alpha,\beta$ and absorb them under big-O notation.
\end{definition}

\section{Main Result}\label{section:main}

Our main result is the existence of low-degree sandwiching polynomials for a wide range of pairs of concept classes and distributions captured by the following assumption.
\begin{assumption}[Valid Instances]\label{assumption:valid-concepts}
    We will consider concept classes $\F$ and target distribution $\D$ over $\R^d$ that satisfy the following properties with parameters $k,d \in \nats,\sigma\ge 1$, $\gamma>0$, where $k\le d$:
    \begin{enumerate}
        \item (Low Intrinsic Dimension) For any $f\in\F$, we have $f:\R^d\to \cube{}$ and there are $F:\R^k\to \cube{}$ and $\mW\in\R^{k\times d}$ with $\mW\mW^\top = \mI_{k\times k}$ such that $f(\x) = F(\mW \x)$, for all $\x\in \R^d$.
        \item (Smooth Boundary) Any $f\in \F$ has $\sigma$-smooth boundary with respect to $\D$.
        \item (Strictly Subexponential Tails) The distribution $\D$ is $\gamma$-strictly subexponential.
    \end{enumerate}
\end{assumption}

The first condition states that every function in the class depends on a potentially different low-dimensional subspace $\cal W$, meaning that its values remain unchanged when the input $\x$ moves in some direction that is perpendicular to the subspace $\cal W$. The second condition ensures that the probability of a sample from $\D$ lying $\rho$-near the decision boundary of $f$ scales linearly with the distance $\rho$, with rate $\sigma$. Finally, the third condition ensures that the distribution $\D$ is sufficiently concentrated in every direction.

We are now ready to state our main theorem.

\begin{theorem}[Main Theorem, restated]\label{theorem:main}
    Let $\F$ be some concept class that satisfies \Cref{assumption:valid-concepts} 
    with parameters $k,d,\sigma$ with respect to a $\gamma$-strictly subexponential 
    distribution $\D$ over $\R^d$. Then, for any $\eps\in(0,1)$ and $s\ge 1$, the 
    $(\eps,s)$-sandwiching degree of $\F$ satisfies
    \[
        \ell(\eps,s)
        \;\le\;
        \widetilde{O}\!\left(
        \frac{\sigma k^{3/2} s}{(\eps/2)^{s+1}}
        \right)^{1+1/\gamma}.
    \]
    Moreover, the sum of the absolute values of the coefficients of the corresponding sandwiching polynomials is upper bounded by $B$ where $\log B= \widetilde{O}(\ell(\eps,s)) \cdot \poly(\log d)$.
\end{theorem}

Note that the above theorem works for any choice of $s\ge 1$, meaning that it establishes the existence of low-degree sandwiching polynomials $\pup,\pdown$ with the property $\|\pup-\pdown\|_{\D,s}\le \eps$, for any choice of $s$. Choosing larger values for $s$ gives stronger notions of approximation, which is sometimes useful for downstream applications. 

For example, the first positive results for TDS learning (\Cref{definition:tds-learning}, see \cite{klivans2023testable}) were based on $\L_2$-sandwiching, although subsequent work showed that $\L_1$-sandwiching suffices \cite{chandrasekaran2024efficient}. For PQ learning (\Cref{definition:pq-learning}), \cite{goel2024tolerant} showed that $\L_2$-sandwiching suffices for efficiency, but it remains an open question whether a similar result can be obtained via $\L_1$-sandwiching. Our work demonstrates that, for a wide range of instances, the gap between the sandwiching degree of different orders is much narrower than previously believed.

Concretely, we obtain the first $\L_2$-sandwiching degree bound for polynomial threshold functions. Hence, we provide the first non-trivial result for PQ learning of the class of PTFs with intrinsic dimension $k$ and degree $q$ such that $q^6k^3 = O(d^{1-c})$ for any $c>0$ with respect to the Gaussian distribution (see \Cref{theorem:sandwiching-polynomial-threshold-functions,theorem:pq-application}).\footnote{For $q^6k^3 = \Omega(d)$, the sandwiching degree becomes $\widetilde\Omega(d)$ (see \Cref{theorem:sandwiching-polynomial-threshold-functions}), and a covering-based algorithm improves upon the result provided by \Cref{theorem:pq-application}.} Note that bounds on the $\L_1$-sandwiching of PTFs with respect to both the Gaussian distribution, as well as the uniform distribution on the boolean hypercube have been provided in prior work \cite{diakonikolas2010ptf,kane2011kindependent,slot2024testably}.

In the rest of this section, we provide the proof of \Cref{theorem:main}.

\subsection{Full-Dimensional Case}\label{section:proof-full-dim}

We will first prove \Cref{theorem:main} in the special case where $k=d$. Throughout this section, we consider $\F,\D$ to be as specified in the premise of \Cref{theorem:main}, with $k=d$. The main idea for the proof is to sandwich an arbitrary element $f$ of $\F$ by functions $\fup,\fdown$ that are Lipschitz and are close to $f$ in expectation. It will then be sufficient to provide sandwiching polynomials for $\fup,\fdown$. The existence of such functions $\fup,\fdown$ is ensured by the following lemma. 

\begin{lemma}\label{lemma:sandwiching-property-smoothed-dilation}
    Let $f$ be a function with $\sigma$-smooth boundary w.r.t. $\D$, $\eps\in(0,1)$ and $s\ge 1$. Then, there exist functions $\fup, \fdown: \R^d\to \R$ that are $L$-Lipschitz for $L = 2\,\sigma\,(\frac{2}{\eps})^s$ such that for any $\x\in\R^d$ and any $s\ge 1$, we have:
    \begin{gather}
        \fdown(\x) \le f(\x) \le\fup(\x) \label{equation:Lipschitz-sandwiching-sandwiching}
        \\
        \|\fup - \fdown\|_{\D,s} \le \eps \label{equation:Lipschitz-sandwiching-approximation}
    \end{gather}
\end{lemma}
\begin{proof}
    We let $\rho = \frac{1}{\sigma}\cdot(\frac{\eps}{2})^s$ and define the following auxiliary sets.
    \begin{gather*}
        \Sin := \{\x\in\R^d: f(\x) = 1\} ,\;\; \Sfarin := \{\x\in \R^d: f^{-\rho}(\x) = 1\}  \\
        \Sout := \{\x\in \R^d: f(\x) = -1\},\;\; \Sfarout := \{\x\in \R^d: f^{+\rho}(\x) = -1\}
    \end{gather*}
    We also define the functions $\gup,\gdown:\R^d\to\R$ as follows.
    \begin{align*}
        \gup(\x) &= \dist(\x,\Sfarout) - \dist(\x,\Sin)  \\
        \gdown(\x) &= \dist(\x,\Sout) - \dist(\x,\Sfarin)
    \end{align*}

    Observe that the functions $\gup,\gdown$ are both $2$-Lipschitz, regardless of the sets $\Sin$, $\Sout$, $\Sfarin$, $\Sfarout$. Therefore, the following choices for $\fup,\fdown$ are $L$-Lipschitz for $L = 2/\rho= 2\sigma(2/\eps)^s$.
    \begin{align*}
        \fup(\x) = \biggr[\frac{\gup(\x)}{\rho}\biggr]_{-1}^{+1} \qquad
        \fdown(\x) = \biggr[\frac{\gdown(\x)}{\rho}\biggr]_{-1}^{+1}
    \end{align*}

    We will show that the following property is true.
    \begin{equation}
        f^{-\rho}(\x) \le \fdown(\x) \le f(\x) \le \fup(\x) \le f^{+\rho}(\x),\;\; \text{ for all }\x\in\R^d
    \end{equation}
    We will focus on $\fup$ since the argument for $\fdown$ is analogous. 
    
    \begin{itemize}
        \item ($f(\x)\le \fup(\x)$). If $f(\x) = 1$, then $\x\in\Sin$ and $\dist(\x,\Sfarout) \ge \rho$. Therefore, $\gup(\x) \ge \rho$ and $\fup(\x) = f(\x) = 1$ for any $\x\in\Sin$. For $\x\not\in \Sin$, we have $f(\x) = -1 \le \fup(\x)$.
        \item ($\fup(\x) \le f^{+\rho}(\x)$). If $f^{+\rho}(\x) = -1$, then $\x\in \Sfarout$ and $\dist(\x,\Sin) \ge \rho$. Therefore, $\gup(\x) \le -\rho$ and $\fup(\x) = -1 = f^{+\rho}(\x)$. If $f^{+\rho}(\x) = 1$, then $\fup(\x)\le 1\le f^{+\rho}(\x)$.
    \end{itemize}

    It remains to bound the quantity $\|\fup-\fdown\|_{\D,s}$. We have the following:
    \begin{align*}
        \| \fup - \fdown \|_{\D,s} &\le \| f^{+\rho} - f^{-\rho}\|_{\D,s} = \Bigr(\E_{\x\sim \D}\Bigr[\Bigr|f^{+\rho}(\x)-f^{-\rho}(\x)\Bigr|^s\Bigr]\Bigr)^{1/s}
    \end{align*}
    Observe that for any $\x$ we have $|f^{+\rho}(\x)-f^{-\rho}(\x)|^s \in\{0,2^s\}$, and therefore we may use the following simplification:
    \[
        |f^{+\rho}(\x)-f^{-\rho}(\x)|^s = 2^s \cdot \frac{f^{+\rho}(\x)-f^{-\rho}(\x)}{2}
    \]
    Therefore, we overall obtain:
    \begin{align*}
        \| \fup - \fdown \|_{\D,s}
        &\le 2\cdot \biggr(\E_{\x\sim \D}\biggr[\frac{f^{+\rho}(\x)-f^{-\rho}(\x)}{2}\biggr]\biggr)^{1/s} \\
        &\le 2\cdot(\sigma \rho)^{1/s} = \eps\,,
    \end{align*}
    where the second inequality follows from \Cref{assumption:valid-concepts}.
\end{proof}

So far, we have obtained a pair of functions that sandwich $f$ and are Lipschitz continuous, but are not necessarily polynomials. To obtain low-degree sandwiching polynomials, we apply the following lemma, first observed in \cite{chandrasekaran2025learning} in the context of learning Lipschitz neural networks under distribution shift. Its proof combines a classical approximation-theoretic result—multivariate Jackson's theorem \cite{Newman1964}—with a bound from \cite{ben2018classical}. In particular, Jackson's theorem guarantees the existence of low-degree $\L_\infty$-approximators for Lipschitz functions over bounded domains, while \cite{ben2018classical} controls the coefficients of such polynomials, ensuring bounded behavior outside the approximation domain.

\begin{lemma}[\cite{chandrasekaran2025learning,Newman1964,ben2018classical}]\label{lemma:uniform-approximators}
    Let $g:\R^d\to \R$ be $L$-Lipschitz and let $\eps\in(0,1)$, $R\ge 1$. Then, there exists polynomial $p$ of degree $\ell = O(LRd/\eps)$ with coefficients bounded by $(d\ell)^{O(\ell)}$ in absolute value and with the following properties:
    \begin{gather}
        |g(\x) - p(\x)| \le \eps, \text{ for all }\x\in\R^d: \|\x\|_2 \le R \\
        |p(\x)| \le (d\ell)^{O(\ell)} \biggr(\frac{\|\x\|_2}{R}\biggr)^{\ell}, \text{ for all }\x\in\R^d: \|\x\|_2 > R
    \end{gather}
\end{lemma}

The polynomials from the above lemma are used in \cite{chandrasekaran2025learning} to obtain a TDS learning algorithm for Lipschitz neural networks under strictly subexponential distributions. To achieve this, they crucially use the fact that the tails of the target distribution are strictly subexponential, to ensure that the expectation of the polynomial approximators outside the approximation domain vanishes for some choice of the parameters $\ell,R$, while still satisfying the condition $\ell=O(LRd/\eps)$. 

In particular, observe that for $\gamma$-strictly subexponential distributions and $p$ as in \Cref{lemma:uniform-approximators}, we have $\E[(p(\x))^2] \le (d\ell)^{O(\ell)}$ and $\pr[\|\x\|_2>R] \le \exp(-\Omega(R/\sqrt{d})^{1+\gamma})$. Therefore: 
\[ \E[|p(\x)| \ind\{\|\x\|_2>R\}] \stackrel{\ell \to \infty}{\to} 0\]

We build on the approach of \cite{chandrasekaran2025learning} by taking the sum $p_1+p_2+\eps$ of the pointwise approximator $p_1$ provided by \Cref{lemma:uniform-approximators} and a closed-form polynomial $p_2$ that vanishes near the origin, and dominates $p_1$ outside the domain of approximation. This ensures that the polynomial $p_1+p_2+\eps$ is an upper sandwiching polynomial for our Lipschitz function $\fup$ (see \Cref{fig:regions-sandwiching}). A symmetric argument provides a lower sandwiching polynomial.

\begin{proof}[Proof of \Cref{theorem:main} for $k = d$.]
    Let $f\in \F$. We will use the functions $\fup,\fdown$ from \Cref{lemma:sandwiching-property-smoothed-dilation}. In particular, it suffices to show the existence of an upper sandwiching polynomial for $\fup$ and a lower sandwiching polynomial for $\fdown$. We focus on $\fup$, since the proof for $\fdown$ will follow from a symmetric argument, i.e., by applying the argument below to $-f$.

    Let $R$ be some large enough parameter which will be chosen later. Let $L = 2\sigma (2/\eps)^s$, and let $p_1$ be a polynomial of degree $\ell_1 = C_1 L R d /\eps$ with the properties specified in \Cref{lemma:uniform-approximators}. We set $\ell_2 = C_2(\ell_1 \log(d\ell_1) + \log(1/\eps))$ and choose the polynomial $\pup$ as follows:
    \begin{equation}
        \pup(\x) = p_1(\x) + p_2(\x) + \eps,\;\;\text{ where } p_2(\x) = \eps \, \biggr( \frac{2\|\x\|_2}{R} \biggr)^{2\ell_2} \label{equation:pup-definition}
    \end{equation}
    For appropriate choices for the universal constants $C_1,C_2\ge 1$ we have the following:
    \begin{itemize}
        \item\label{item:unif-approx-region} For any $\x$ with $\|\x\|_2 \le R/2$, we have $\pup(\x) \le p_1(\x) + 2\eps \le \fup(\x) + 3\eps$ and $\pup(\x) \ge \fup(\x)$, since $|p_1(\x)-\fup(\x)| \le \eps$ and $p_2(\x)\in[0,\eps]$.
        \item For any $\x$ with $\|\x\|_2 \in (R/2, R]$ we have $\pup(\x) \ge \fup(\x)$, since $|p_1(\x)-\fup(\x)| \le \eps$.
        \item For any $\x$ with $\|\x\|_2 > R$ we have $\pup(\x) \ge \fup(\x)$, since $p_2(\x) \ge 1+|p_1(\x)|$.
    \end{itemize}
    So far, we have established that $\fup(\x)\le \pup(\x)$ for all $\x\in\R^d$. We will now show that the quantity $\|\pup-\fup\|_{\D,s}$ is bounded by $O(\eps)$ for some appropriate choice of $R = \widetilde{O}((Ls / \eps)^{1/\gamma} d^{0.5+ 1.5/{\gamma}})$.\footnote{If we wish to achieve error $\eps$ instead of $O(\eps)$, we may just substitute $\eps' = \eps/C$ for some universal constant $C\ge 1$.}
    
    Due to the properties of $\pup$ in the region $\|\x\|_2\le R/2$, it suffices to account for the region $\|\x\|_2> R/2$. In particular, we have:
    \begin{align}
        \|\pup-\fup\|_{\D,s} &\le 3\eps + \Bigr( \E_{\x\sim \D}\Bigr[ (\pup(\x)-\fup(\x))^s \ind\{\|\x\|_2 > R/2\}\Bigr] \Bigr)^{1/s} \notag\\
        &\le 3\eps + \Bigr( \E_{\x\sim \D}\Bigr[ (\pup(\x)-\fup(\x))^{2s} \Bigr] \Bigr)^{\frac{1}{2s}} \Bigr( \pr_{\x\sim \D}\Bigr[ \|\x\|_2 > R/2\Bigr] \Bigr)^{\frac{1}{2s}} \notag \\
        &= 3\eps + \|\pup-\fup\|_{\D,2s} \,\cdot\, \Bigr( \pr_{\x\sim \D}\Bigr[ \|\x\|_2 > R/2\Bigr] \Bigr)^{\frac{1}{2s}}\label{equation:pup-fup-bound}
    \end{align}
    where the second inequality follows from the Cauchy-Schwarz inequality.

    Since the distribution $\D$ is $\gamma$-strictly subexponential, we have the following for some constants $\alpha,\beta>0$:
    \begin{align}
        \pr_{\x\sim \D}\Bigr[ \|\x\|_2 > R/2\Bigr] &\le \pr_{\x\sim \D}\Bigr[ \|\x\|_\infty > R/{(2\sqrt{d}})\Bigr] \notag \\
        &\le d\cdot \sup_{\w\in\S^{d-1}}\pr_{\x\sim\D}\Bigr[ |\w\cdot \x| > R/{(2\sqrt{d}})\Bigr] \notag \\
        &\le \alpha\, d\cdot \exp\biggr(-\beta \Bigr( \frac{R}{2\sqrt{d}}\Bigr)^{1+\gamma} \biggr) \label{equation:subexp-radius-tail}
    \end{align}
    By combining the above tail bound with the bounds on $|p_1(\x)|$ and $|p_2(\x)|$ (see \Cref{lemma:uniform-approximators} and \eqref{equation:pup-definition}), we also obtain the following bound:
    \begin{equation}
        \|\pup-\fup\|_{\D,2s} \le (d\ell_2)^{O(d\ell_2)} = \exp\biggr( \widetilde{O}\biggr( \frac{LRd}{\eps} \biggr) \biggr) \label{equation:pup-fup-bound-crude}
    \end{equation}
    Therefore we may combine \eqref{equation:pup-fup-bound}, \eqref{equation:subexp-radius-tail}, and \eqref{equation:pup-fup-bound-crude} to obtain that there is some appropriate choice for $R$ with $R = \widetilde{O}((Ls / \eps)^{1/\gamma} d^{0.5+ 1.5/{\gamma}})$ such that $\|\pup-\fup\|_{\D,s} \le O(\eps)$. The overall bound on the degree of $\pup$ is: \[\deg(\pup) = 2\ell_2 = \widetilde{O}\!\left(
        \frac{\sigma d^{3/2} s}{(\eps/2)^{s+1}}\right)^{1+1/\gamma}\,,\]
    {     as desired.}
\end{proof}

\subsection{Concepts with Low Intrinsic Dimension}

To complete the proof of \Cref{theorem:main}, it remains to account for the case where $k\le d$. To this end, we show that for a function $f(\x) = F(\mW\x)$ where $\mW\mW^\top = \mI_{k\times k}$, the boundary smoothness parameter of $f$ is equal to the boundary smoothness parameter associated with $F$. Our result then follows from the observation that for any polynomial $p$, the function $p(\mW\x)$ is a polynomial of the same degree.

\begin{proposition}
    Let $f\in\F$, where $\F$ satisfies \Cref{assumption:valid-concepts} with parameters $k,d,\sigma$ with respect to a distribution $\D$ over $\R^d$. Let $F:\R^k\to \cube{}$ be such that $f(\x) = F(\mW\x)$ where $\mW\in \R^{k\times d}$ and $\mW\mW^\top = \mI_{k\times k}$. Then $F$ has $\sigma$-smooth boundary with respect to the distribution of $\mW\x$ when $\x\sim \D$.
\end{proposition}

\begin{proof}
    It suffices to show that $f^{\pm\rho}(\x) = F^{\pm\rho}(\mW\x)$ for all $\x\in\R^d$. We have the following:
    \begin{align*}
        f^{+\rho}(\x) &= \sup_{\z\in\R^d:\|\z\|_2 \le \rho} f(\x+\z) \\
        &= \sup_{\z\in\R^d:\|\z\|_2 \le \rho} F(\mW\x+\mW\z) \\
        &= \sup_{\vt\in\R^k:\|\vt\|_2 \le \rho} F(\mW\x+\vt) = F^{+\rho}(\mW\x)\,,
    \end{align*}
    where we have used the fact that the image of the unit ball in $d$ dimensions under $\mW$ is the unit ball in $k$ dimensions. Similarly, we can show that $f^{-\rho}(\x) = F^{-\rho}(\mW\x)$.
\end{proof}

\begin{proof}[Proof of \Cref{theorem:main}.] 
    In \Cref{section:proof-full-dim} we proved \Cref{theorem:main} in the special case where $k=d$. 
    Let $f\in\F$ with $f(\x)=F(\mW\x)$. By the $k=d$ case, there exist $(\eps,s)$-sandwiching polynomials 
    $P_{\mathrm{up}},P_{\mathrm{down}}$ for $F$ with respect to the distribution of $\mW\x$ when $\x\sim\D$, 
    with the desired degree bound $\ell$.

    Define $\pup(\x)=P_{\mathrm{up}}(\mW\x)$ and $\pdown(\x)=P_{\mathrm{down}}(\mW\x)$. 
    Then $\pup,\pdown$ are $(\eps,s)$-sandwiching polynomials for $f$ with respect to $\D$. Moreover, we have $\|\pup\|_\coef,\|\pdown\|_\coef \le (d\ell)^{{O}(\ell)}$, because the coefficients of the linear polynomial $p_{\mathrm{lin}}(\x) = \mW\x$ satisfy $\|p_{\mathrm{lin}}\|_\coef \le \poly(d)$, and we also have $\|\pup\|_\coef \le \poly(\|P_{\mathrm{up}}\|_\coef, (\|p_{\mathrm{lin}}\|_\coef)^{O(\deg(P_{\mathrm{up}}))})$, where $\|P_{\mathrm{up}}\|_\coef \le (k\ell)^{O(\ell)}$.
\end{proof}

\section{Sandwiching Degree Bounds for Fundamental Concept Classes}\label{section:sandwiching-bounds-for-classes}

In this section, we provide a suite of new or improved sandwiching degree bounds for several important concept classes. 

One especially appealing property of boundary smoothness is that it behaves additively under composition. In particular, the following proposition shows that composing an arbitrary Boolean function with concepts of smooth boundary preserves boundary smoothness, with a smoothness parameter equal to the sum of the individual parameters.

\begin{proposition}\label{proposition:composition-of-boundary-smoothness}
    Let $g_1,g_2,\dots,g_k:\R^d \to \cube{}$ be such that $g_i$ has $\sigma_i$-smooth boundary with respect to a distribution $\D$ over $\R^d$. Then, for any $F:\R^k \to \cube{}$, the function $f(\x) = F(g_1(\x),g_2(\x),\dots,g_k(\x))$ has $\sum_{i\in[k]}\sigma_i$-smooth boundary with respect to $\D$.
\end{proposition}

\begin{proof}
    Let $\x\in\R^d$ such that $f^{+\rho}(\x) \neq f^{-\rho}(\x)$. Then, there must be some $\z\in\R^d$ with $\|\z\|_2\le \rho$ and some $i\in[k]$ such that $g_i(\x)\neq g_i(\x+\z)$. Moreover, we have $g_i^{+\rho}(\x)-g_i^{-\rho}(\x) \in \{0,2\}$ for all $\x$. Therefore, we have:
    \begin{align*}
        \E_{\x\sim\D}\biggr[ \frac{f^{+\rho}(\x) - f^{-\rho}(\x)}{2} \biggr] &\le \sum_{i\in[k]} \E_{\x\sim\D}\biggr[ \frac{g_i^{+\rho}(\x) - g_i^{-\rho}(\x)}{2} \biggr] \le \sum_{i\in[k]}\sigma_i\,,
    \end{align*}
    { as desired.}
\end{proof}
As we shall see, several fundamental concept classes—including intersections of halfspaces, low-dimensional polynomial threshold functions, and low-dimensional convex sets—have smooth boundary. Together with \Cref{theorem:main}, the above proposition yields sandwiching degree bounds not only for each of these classes individually, but also for arbitrary Boolean combinations thereof.

\subsection{Functions of Halfspaces}

We begin by focusing on the fundamental class of halfspaces. The decision boundary of a halfspace is a hyperplane. Therefore, the probability of falling near the boundary of a halfspace is equal to the probability of falling within a narrow band. As such, the smooth boundary condition for halfspaces is satisfied precisely for distributions that are anticoncentrated in the following sense.

\begin{definition}[Anticoncentration]\label{definition:anticoncentration}
    For $\alpha>0$, we say that a distribution $\D$ over $\R^d$ is $\alpha$-anticoncentrated in every direction if for any $\w\in\S^{d-1}$, any $t\in\R$, and any $r\ge 0$,
    \[
        \pr_{\x\sim\D}[|\w\cdot \x - t|\le r]\le \alpha\, r\,.
    \]
    In what follows, we may omit the parameter $\alpha$ and absorb it into universal constants under big-O notation.
\end{definition}

The following proposition is immediate from the definition of anticoncentration.

\begin{proposition}\label{proposition:halfspaces-smooth-boundary}
    Let $\D$ be some distribution over $\R^d$ that is anticoncentrated in every direction.
    Then, the class of halfspaces has $O(1)$-smooth boundary with respect to $\D$.
\end{proposition}

\begin{proof}
    Let $f(\x) = \sign(\w\cdot \x-t)$ for some $\w\in\S^{d-1}$ and $t\in\R$. Then, we have:
    \begin{align*}
        \E_{\x\sim\D}\biggr[ \frac{f^{+\rho}(\x) - f^{-\rho}(\x)}{2} \biggr] &= \pr_{\x\sim\D}[|\w\cdot \x - t|\le \rho] \le O(1)\cdot \rho\,,
    \end{align*}
    { as desired.}
\end{proof}

\paragraph{General Functions of Halfspaces} The following result provides an exponential improvement over the best known bound for the $(\eps=0.1, s=1)$-sandwiching degree of functions of $k$ halfspaces with respect to any anticoncentrated and strictly subexponential distribution \cite{gopalan2010fooling,gollakota2022moment}.

\begin{theorem}\label{theorem:sandwiching-functions-of-halfspaces}
    Let $\F$ be the class of arbitrary functions of $k$ halfspaces over $\R^d$ and let $\D$ be some distribution over $\R^d$ that is anticoncentrated in every direction and $\gamma$-strictly subexponential. Then, the $(\eps,s)$-sandwiching degree of $\F$ with respect to $\D$ is 
    \[
        \ell(\eps,s) \le \widetilde{O}\biggr(\frac{k^{5/2}s}{(\eps/2)^{s+1}}\biggr)^{1+1/\gamma}\,.
    \]
\end{theorem}

The above theorem is an immediate corollary of \Cref{theorem:main} and the following lemma, which is a combination of \Cref{proposition:composition-of-boundary-smoothness,proposition:halfspaces-smooth-boundary}.

\begin{lemma}[Combination of \Cref{proposition:composition-of-boundary-smoothness,proposition:halfspaces-smooth-boundary}]\label{lemma:halfspace-functions-boundary-smoothness}
    Let $\F$ be the class of arbitrary functions of $k$ halfspaces over $\R^d$ and let $\D$ be some distribution over $\R^d$ that is anticoncentrated in every direction. Then, $\F$ has $O(k)$-smooth boundary with respect to $\D$. 
\end{lemma}

Note that \Cref{theorem:sandwiching-functions-of-halfspaces} also provides a polynomial improvement over the previous state-of-the-art upper bound on the sandwiching degree of halfspace intersections with respect to the Gaussian distribution for constant $\eps$ \cite{gopalan2010fooling,klivans2023testable}. In particular, the best known bound before this work was $O(k^6)$, and \Cref{theorem:sandwiching-functions-of-halfspaces} gives a bound of $\widetilde{O}(k^5)$ for the Gaussian distribution ($\gamma = 1$). 

\paragraph{Improvements for Intersections of Halfspaces under the Gaussian} In the following theorem, we show that we can actually improve the sandwiching bound of $k$-halfspace intersections with respect to the Gaussian distribution to $\widetilde{O}(k^3)$.

\begin{theorem}\label{theorem:sandwiching-intersections-of-halfspaces}
    Let $\F$ be the class of intersections of $k$ halfspaces over $\R^d$. Then, the $(\eps,s)$-sandwiching degree of $\F$ with respect to the standard Gaussian distribution is 
    \[
        \ell(\eps,s) \le \widetilde{O}\biggr(\frac{k^{3}s}{(\eps/2)^{2s+2}}\biggr)\,.
    \]
\end{theorem}

The above theorem is a corollary of \Cref{theorem:main} combined with the following result which gives a sharp bound on the boundary smoothness parameter of halfspace intersections.

\begin{lemma}[Boundary Smoothness of Halfspace Intersections]\label{lemma:intersections-smooth-boundary}
    Let $\F$ be the class of intersection of $k$ halfspaces over $\R^d$. Then, $\F$ has $O(\sqrt{\log k})$-smooth boundary with respect to the standard Gaussian distribution.
\end{lemma}

At the core of \Cref{lemma:intersections-smooth-boundary} is the following bound follows from a result by Nazarov on the Gaussian surface area of halfspace intersections (see \cite{nazarov2003maximal,klivans2008learning}) as well as the fact that halfspace intersections have $(d-1)$-rectifiable boundaries. See \Cref{section:GSA-rectifiable} and \cite{kane2014average,de2023gaussian} for relevant discussions.

\begin{fact}[Nazarov (see also \cite{nazarov2003maximal,klivans2008learning})]\label{fact:nazarov-gsa-bound-intersections}
    Let $f:\R^d\to \cube{}$ be an intersection of $k$ halfspaces. Let $\partial f $ be the boundary of $f$, i.e., the set of points $\x$ where for any $\rho>0$ there exists $\z\in\R^d$ with $\|\z\|_2\le \rho$ and $f(\x)\neq f(\x+\z)$. Then, we have the following:
    \[
        \int_{\partial f} \phi(\x)\,d\H^{d-1}(\x) \le \sqrt{2\ln k} + 2\,,
    \]
    where $\phi$ is the standard Gaussian density $\phi(\x) = (2\pi)^{-d/2} \exp(-\|\x\|_2^2/2)$ and $\H^{d-1}$ is the $(d-1)$-dimensional Hausdorff measure.
\end{fact}

Nazarov's result gives an asymptotic version of \Cref{lemma:intersections-smooth-boundary}, i.e., it shows that for an intersection of $k$ halfspaces $f$ we have:
\[
    \lim_{\rho\to 0^+} \E\biggr[\frac{f^{+\rho}(\x)- f^{-\rho}(\x)}{2\rho} \biggr] \le O(\log k)
\]
To obtain a non-asymptotic version of this result, we use the coarea formula, which is an integration tool from geometric measure theory \cite{federer1969geometric}. To apply the coarea formula properly, we define the following alternative notions of dilation and erosion which are specialized to halfspace intersections. Crucially, these operations generate functions that themselves are intersections of the same number of halfspaces.

\begin{definition}[Bias-shift Dilation and Erosion]\label{definition:bias-shift-dilation-erosion}
        Let $f:\R^d\to \cube{}$ be an intersection of $k$ halfspaces, i.e., $f(\x) = 2\cdot \prod_{i=1}^k \ind\{\w_i\cdot \x\le \tau_i\} - 1 $ for some $\w_i\in\S^{d-1}, \tau_i\in \R$. For $\rho \ge 0$, we define the $\rho$-bias-shift dilation $f^{\boxplus\rho}$, as well as the $\rho$-bias-shift erosion $f^{\boxminus\rho}$ of $f$ as follows.
    \begin{align*}
        f^{\boxplus\rho}(\x) &= 2\cdot \prod_{i=1}^k \ind\{\w_i\cdot \x\le \tau_i+\rho\} - 1
        \\ f^{\boxminus\rho}(\x) &= 2\cdot \prod_{i=1}^k \ind\{\w_i\cdot \x\le \tau_i-\rho\} - 1
    \end{align*}
\end{definition}

Moreover, in the following proposition we show that the expected difference between $f^{+\rho}$ and $f^{-\rho}$ is dominated by the difference $f^{\boxplus\rho}-f^{\boxminus\rho}$. This is important for our analysis, because we will essentially integrate the surface integrals of $f^{\boxplus t}, f^{\boxminus t}$ over $t \in [0,\rho]$, and will use the fact that these functions are halfspace intersections to invoke \Cref{fact:nazarov-gsa-bound-intersections}.

\begin{proposition}\label{proposition:bias-shift-dilation-vs-regular-dilation}
    Let $f:\R^d\to \cube{}$ be an intersection of $k$ halfspaces. We have: 
    \[ f^{\boxminus\rho}(\x) = f^{-\rho}(\x)  \le f(\x) \le f^{+\rho}(\x) \le f^{\boxplus\rho}(\x)\,, \text{ for all } \x\in\R^d\text{ and } \rho\ge 0\]
\end{proposition}

\begin{proof}
    Observe that for all $i\in[k]$ we have that $|\w_i\cdot \z| \le \|\z\|_2$. Therefore, we have the following:
    \begin{align*}
        f^{+\rho}(\x) = \sup_{\z:\|\z\|_2\le \rho} f(\x+\z) &= \sup_{\z:\|\z\|_2\le \rho} 2\cdot \prod_{i=1}^k \ind\{\w_i\cdot \x\le \tau_i-\w_i\cdot \z\} - 1 \\
        &\le \sup_{\z:\|\z\|_2\le \rho} 2\cdot \prod_{i=1}^k \ind\{\w_i\cdot \x\le \tau_i+\rho\} - 1 \\ &= f^{\boxplus \rho}(\x)
    \end{align*}
    We have shown that $f^{+\rho} \le f^{\boxplus \rho}$. A symmetric argument gives $f^{-\rho}\ge f^{\boxminus\rho}$. Moreover, in the case of erosion, we obtain $f^{-\rho}= f^{\boxminus\rho}$. To see this, let $\z^*=\z^*(\x) = \rho\, \w_{i^*(\x)}$ where $i^*(\x) = \argmax_{i\in[k]} \{\w_i\cdot \x-\tau_i\}$. Suppose that $f(\x) = 1$, since the other case is trivial. We have:
    \begin{align*}
        f^{-\rho}(\x) &= 2\cdot \inf_{\z:\|\z\|_2\le \rho} \ind\{\w_{i^*}\cdot (\x+\z) \le \tau_{i^*}\} - 1 \\
        &= 2\cdot\ind\{\w_{i^*}\cdot (\x+\z^*) \le \tau_{i^*}\}-1 \\
        &= 2\cdot\ind\{\w_{i^*}\cdot \x \le \tau_{i^*} - \rho\}-1 \\
        &= f^{\boxminus\rho}(\x)\,,
    \end{align*}
    since the coordinate $i^*$ corresponds to the smallest gap between $\w_{i}\cdot \x$ and $\tau_{i}$ among all $i\in[k]$.
\end{proof}

We are now ready to prove \Cref{lemma:intersections-smooth-boundary}.

\begin{proof} of \Cref{lemma:intersections-smooth-boundary}. 
    Let $f$ be an intersection of $k$ halfspaces, i.e., $f(\x) = 2\cdot \prod_{i=1}^k \ind\{\w_i\cdot \x\le \tau_i\} - 1 $ for some $\w_i\in\S^{d-1}$ and $\tau_i \in \R$. Due to \Cref{proposition:bias-shift-dilation-vs-regular-dilation}, we have the following for any $\rho\ge 0$:
    \[
        \E_{\x\sim\Gauss_d}\Bigr[\frac{f^{+\rho}(\x)-f^{-\rho}(\x)}{2}\Bigr] \le \E_{\x\sim\Gauss_d}\Bigr[\frac{f^{\boxplus\rho}(\x)-f^{\boxminus\rho}(\x)}{2}\Bigr]
    \]
    It is now convenient to define the following slack function:
    \[
        \Psi(\x) = \max_{i\in[k]} \{\w_i\cdot \x-\tau_i\}
    \]
    Using this notation, we have:
    \begin{align*}
        \E_{\x\sim\Gauss_d}\Bigr[\frac{f^{\boxplus\rho}(\x)-f^{\boxminus\rho}(\x)}{2}\Bigr] &= \pr_{\x\sim\Gauss_d}\Bigr[\Psi(\x)\in[-\rho,\rho]\Bigr] = \int_{\x: |\Psi(\x)|\le \rho} \phi(\x)\,d\x
    \end{align*}
    Observe that the function $\Psi$ is $1$-Lipschitz as a maximum of affine functions. Hence, apart from a measure-zero set $\mathsf{ND}_{\Psi}$, we have $\|\nabla \Psi(\x)\|_2 = 1$, because it equals $\|\w_i\|_2 $ for an active constraint $i$. We may apply the coarea formula (Theorem 3.2.22 in \cite{federer1969geometric}) to obtain:
    \begin{align*}
        \int_{\x:|\Psi(\x)|\le \rho} \phi(\x)\,d\x &= \int_{t=-\rho}^{\rho} \int_{\x\in \Psi^{-1}(t)\setminus\mathsf{ND}_\Psi} \frac{\phi(\x)}{\|\nabla \Psi(\x)\|_2} \,d\H^{d-1} \, dt \\
        &= \int_{t=-\rho}^{\rho} \int_{\x\in \Psi^{-1}(t)} {\phi(\x)} \,d\H^{d-1} \, dt
    \end{align*}
    Observe now that, for $t\in[-\rho,\rho]$ with $r = |t|$, the set $\Psi^{-1}(t) = \{\x:\Psi(\x) = t\}$ is either equal to $\partial f^{\boxplus r}$ or $\partial f^{\boxminus r}$ depending on whether $t\ge 0$ or $t<0$ (see \Cref{definition:boundary}). The function $f^{\boxplus r}$ is always an intersection of halfspaces. The function $f^{\boxminus r}$ is either an intersection of halfspaces, or $\partial f^{\boxminus r} = \emptyset$. Therefore, by applying \Cref{fact:nazarov-gsa-bound-intersections}, we obtain:
    \[
        \int_{\x: |\Psi(\x)|\le \rho} \phi(\x)\,d\x = \int_{t=-\rho}^\rho \bigr(\sqrt{2\ln k} + 2\bigr) \, dt = O\bigr(\sqrt{\log k}\bigr) \, \rho \,,
    \]
    which concludes the proof of \Cref{lemma:intersections-smooth-boundary}.
\end{proof}

\subsection{Convex Sets and Polynomial Threshold Functions}

\paragraph{Low-Dimensional Convex Sets}
Another important geometric concept class is that of low-dimensional convex sets. While the boundary smoothness of convex sets has been studied in prior work \citep{nazarov2003maximal,chandrasekaran2024efficient}, there have been no end-to-end results on the sandwiching degree of convex sets. Previous work \citep{de2023gaussian} showed that convex sets in $k$ dimensions under the Gaussian distribution admit upper sandwiching functions representable as intersections of $2^{O(k)}$ halfspaces. Combined with the sandwiching bounds for intersections of halfspaces from \cite{gopalan2010fooling}, this implies the existence of upper sandwiching polynomials for convex sets, albeit of exponential degree. While this approach could in principle be extended to obtain lower sandwiching polynomials as well, our result achieves an exponential improvement by establishing sandwiching degree $\poly(k)$ for convex sets under the Gaussian distribution.

\begin{theorem}\label{theorem:sandwiching-convex-sets}
    Let $\F$ be the class of convex functions in $\R^d$ with intrinsic dimension $k$. Then, the $(\eps,s)$-sandwiching degree of $\F$ with respect to $\Gauss(0,\mI_{d\times d})$ is 
    \[
        \ell(\eps,s) \le \widetilde{O}\biggr(\frac{k^{5/2}s}{(\eps/2)^{s+1}}\biggr)^{1+1/\gamma}\,.
    \]
\end{theorem}

The above theorem follows from combining \Cref{theorem:main} with the following result from \cite{chandrasekaran2024efficient}.

\begin{lemma}[\cite{chandrasekaran2024efficient}]
    The class of convex functions in $\R^d$ with intrinsic dimension $k$ has $\sigma$-smooth boundary for $\sigma = O(k \log k)$.
\end{lemma}

\paragraph{Low-Dimensional PTFs} The last class we consider is that of polynomial threshold functions. Recent work by \cite{slot2024testably} building on the approach of \cite{kane2011kindependent} gave a bound on the sandwiching degree of degree-$q$ PTFs that does not depend on the intrinsic dimension $k$, but the dependence on $q$ is doubly exponential or worse.\footnote{The dependence on $q$ is not made explicit in prior work, as $q$ is considered to be a constant.} In the following bound, we achieve an exponential improvement over prior work, as long as $q=\Omega(\log k)$, or a doubly-exponential improvement if $q = \poly(k)$.

\begin{theorem}\label{theorem:sandwiching-polynomial-threshold-functions}
    Let $\F$ be the class of degree-$q$ PTFs in $\R^d$ with intrinsic dimension $k$. Then, the $(\eps,s)$-sandwiching degree of $\F$ with respect to $\Gauss(0,\mI_{d\times d})$ is 
    \[
        \ell(\eps,s) \le \widetilde{O}\biggr(\frac{q^3 k^{5/2}s}{(\eps/2)^{s+1}}\biggr)^{1+1/\gamma}\,.
    \]
\end{theorem}

Once more, our result follows from \Cref{theorem:main} combined with a corresponding bound on the boundary smoothness parameter of PTFs from \cite{chandrasekaran2024efficient}.

\begin{lemma}[\cite{chandrasekaran2024efficient}]
    The class of degree-$q$ PTFs in $\R^d$ with intrinsic dimension $k$ has $\sigma$-smooth boundary for $\sigma = O(q^3k)$.
\end{lemma}

\section{Applications}\label{section:applications}

Sandwiching polynomials are a strong and versatile tool. While standard approximation theory provides low-degree polynomials that approximate functions in expectation, sandwiching polynomials additionally constrain the approximated function to lie within a narrow envelope. Crucially, this guarantee holds pointwise over the entire domain, rather than merely in expectation with respect to a given distribution. As a result, sandwiching polynomials remain representative of the target function even under distributions other than the one defining the approximation-in-expectation guarantee. This property has recently led to several concrete applications of sandwiching in learning theory, including learning under distribution shift \citep{klivans2023testable,chandrasekaran2024efficient}, learning with heavy contamination \citep{klivans2025the}, and testable learning \citep{rubinfeld2022testing,gollakota2022moment}.

In what follows, we describe a collection of implications of our theorem for learning theory, as well as for pseudorandomness.

\subsection{Learning Theory}\label{section:learning-applications}

A recent line of work in learning theory has focused on frameworks for algorithms that generate certificates of their own performance \citep{goldwasser2020beyond,rubinfeld2022testing,klivans2023testable}. This line of work emerged in response to the strong assumptions that are ubiquitous in learning theory and are often either unverifiable from samples or computationally intractable to check. Across these frameworks, it has been shown that the existence of low-degree sandwiching polynomials suffices for efficient learnability \citep{gollakota2022moment,klivans2023testable,goel2024tolerant,chandrasekaran2024efficient}.

Moreover, the techniques developed in this line of work have been used to obtain near-optimal upper bounds \citep{pmlr-v291-klivans25a,klivans2025the} on the computational complexity of the classical problem of learning under contamination \citep{Valiant85,KearnsL93,BSHOUTY2002255}. In fact, leveraging sandwiching polynomials allows these results to extend even to the setting of heavy contamination \citep{klivans2025the}.

All of our learning-theoretic applications follow from combining \Cref{theorem:main} with existing results showing that sufficiently low sandwiching degree implies efficient learnability under the corresponding learning primitive. These results additionally require uniform convergence for the relevant class of polynomials under the target marginal distribution $\D$. Since the sandwiching polynomials provided by \Cref{theorem:main} have bounded coefficients and bounded degree, and since we restrict attention to (strictly) subexponential distributions, the required uniform convergence guarantees hold in our setting, as discussed in \Cref{section:concentration}.

In the following, we will use the notation $\bar\D$ to denote a labeled distribution over $\R^d\times\cube{}$, where $\D$ is the marginal of $\bar\D$ on $\R^d$. We will also use a similar convention for labeled sets of examples $\bar S$ and their unlabeled counterparts $S$. 

\paragraph{Testable Learning} We first focus on the problem of testable agnostic learning, in which the learner receives labeled samples from an unknown distribution over $\R^d\times\cube{}$ without any guaranteed assumptions on the underlying data-generating process. While there is strong evidence that some structural assumptions on the $\R^d$-marginal are necessary for efficient learnability even for simple concept classes \citep{daniely2016complexity,diakonikolas2022near,diakonikolas2022cryptographic}, testable learning algorithms avoid these hardness barriers by retaining the ability to abstain when their target distributional assumption is violated. In particular, such algorithms either accept and output a hypothesis whose error is near-optimal with respect to a given concept class, or reject by declaring that the underlying assumption does not hold.

\begin{definition}[Testable Learning \citep{rubinfeld2022testing}]\label{definition:testable-learning}
    An algorithm $\A$ is a tester-learner for concept class $\F$ with respect to some target distribution $\D^*$ over ${\R^d}$ if on input $(\epsilon,\delta,\bar S)$, where $\epsilon,\delta\in(0,1)$ and $\bar S$ is a set of i.i.d. examples from some arbitrary distribution $\Dlabeled$, the algorithm $\A$ either outputs $\mathrm{Reject}$ or outputs $(\mathrm{Accept},h)$, where $h:{\R^d}\to \cube{}$ such that with probability at least $1-\delta$ over $\bar S$, and the randomness of $\A$, the following conditions hold. 
    \begin{enumerate}
        \item (\textit{Soundness}) Upon acceptance, the error of $h$ is bounded as follows: 
        \[
            \pr_{(\x,y)\sim \Dlabeled}[y \neq h(\x)] \le \min_{\concept\in\F}\pr_{(\x,y)\sim \Dlabeled}[y\neq \concept(\x)]+\eps
        \]
        \item (\textit{Completeness}) If $\D=\Dtarget$, where $\D$ is the marginal of $\Dlabeled$ on ${\R^d}$, then $\A$ accepts.
    \end{enumerate}
\end{definition}

The work of \cite{gollakota2022moment} showed that low sandwiching degree implies efficient algorithms for testable learning. Therefore, \Cref{theorem:main} implies the following result.

\begin{theorem}[Combination of \cite{gollakota2022moment} with \Cref{theorem:main}]\label{theorem:testable-application}
    Let $\F$ be some concept class that satisfies \Cref{assumption:valid-concepts} 
    with parameters $k,d,\sigma$ with respect to a $\gamma$-strictly subexponential 
    distribution $\D^*$ over $\R^d$. There is a tester-learner for $\F$ with respect to $\D^*$ that has runtime and sample complexity $d^{\widetilde{O}(\ell)}O(\log 1/\delta)$ where we have:
    \[
        \ell
        \;\le\;
        \widetilde{O}\!\left(
        \frac{\sigma k^{3/2}}{\eps^{2}}
        \right)^{1+1/\gamma}.
    \]
\end{theorem}

One limitation of \Cref{definition:testable-learning} for testable agnostic learning is that it permits the learner to reject whenever there is even a small violation of the target assumption $\D^*$. As a result, the corresponding tester-learners may be fragile and reject too often, which can limit their practical usefulness. Fortunately, it is possible to design algorithms that are guaranteed to accept with high probability even when the distributional assumption is violated, provided that the marginal distribution is sufficiently close to the target $\D^*$ in total variation distance. Such algorithms are known as tolerant tester-learners and are defined as follows.

\begin{definition}[Tolerant Testable Learning \citep{goel2024tolerant}]\label{definition:tolerant-testable-learning}
    The definition is the same as for testable learning (\Cref{definition:testable-learning}), with the only difference being that the algorithm receives an additional parameter $\tau\in(0,1)$ as input and the soundness and completeness criteria are modified as follows:
    \begin{enumerate}
        \item (\textit{Soundness}) Upon acceptance, the error of $h$ is bounded as follows: 
        \[
            \pr_{(\x,y)\sim \Dlabeled}[y \neq h(\x)] \le \min_{\concept\in\F}\pr_{(\x,y)\sim \Dlabeled}[y\neq \concept(\x)]+\tau+\eps
        \]
        \item (\textit{Completeness}) If $\mathrm{d}_{\mathrm{TV}}(\D,\Dtarget)\le \tau$, where $\D$ is the marginal of $\Dlabeled$ on ${\R^d}$, then $\A$ accepts.
    \end{enumerate}
\end{definition}

To obtain upper bounds on the computational complexity of tolerant testable learning, an additional distributional assumption on the target marginal $\D^*$ is required. In particular, \cite{klivans2025the} assume that $\D^*$ is hypercontractive with respect to arbitrary polynomials of any degree. This property is defined below and is satisfied by a broad class of distributions, including all log-concave distributions \citep{bobkov2001some,saumard2014log}, as well as the uniform distribution on the Boolean hypercube.

\begin{definition}[Hypercontractivity]\label{definition:hypercontractivity}
    A distribution $\D^*$ over $\R^d$ is polynomially hypercontractive if there is a constant $C\ge 1$ such that for any polynomial $p$ over $\R^d$ and any $q \ge 2$ we have
    \begin{enumerate}
        \item $\E_{\x\sim \D^*}[|p(\x)|^q] \le ({C}q)^{\ell q} \bigr(\E_{\x\sim \D^*}[|p(\x)|]\bigr)^q\,, \text{ where }\ell \text{ is the degree of }p$.
        \item The absolute expectations of degree-$1$ monomials are finite under $\D^*$.
    \end{enumerate}
\end{definition}

We obtain the following result by combining a theorem from \cite{klivans2025the}---which shows that $\L_1$ sandwiching suffices for tolerant testable learning---with \Cref{theorem:main}.

\begin{theorem}[Combination of \cite{klivans2025the} with \Cref{theorem:main}]\label{theorem:tolerant-testable-application}
    Let $\F$ be some concept class that satisfies \Cref{assumption:valid-concepts} 
    with parameters $k,d,\sigma$ with respect to a $\gamma$-strictly subexponential and polynomially hypercontractive 
    distribution $\D^*$ over $\R^d$. There is a tester-learner for $\F$ with respect to $\D^*$ that has runtime and sample complexity $d^{\widetilde{O}(\ell)}O(\log 1/\delta)$ where we have:
    \[
        \ell
        \;\le\;
        \widetilde{O}\!\left(
        \frac{\sigma k^{3/2}}{\eps^{2}}
        \right)^{1+1/\gamma}.
    \]
\end{theorem}

\paragraph{Learning with Distribution Shift} Another learning problem in which sandwiching polynomials play a central role is learning under distribution shift. In this setting, the learner has access to labeled samples from a training distribution, as well as unlabeled samples from a potentially different test distribution. The goal is to output a hypothesis with low error on the test distribution. Since labels from the test distribution are not observed, this task is only feasible when the training and test labeling functions are suitably related. The following quantity---originally introduced in the domain adaptation literature \citep{ben2006analysis,mansour2009domadapt,ben2010theory}---quantifies this relationship in terms of a hypothesis that simultaneously achieves low error on both the training and test distributions.

\begin{definition}[Error Benchmark in the Distribution Shift Setting]\label{definition:lambda}
    Let $\bar\D, \bar\D'$ be two distributions over ${\R^d}\times\cube{}$ and $\F$ some concept class over $\R^d$. We consider the following error benchmark:
    \[
        \lambda(\bar\D,\bar\D';\F) = \min_{\concept\in\F}\Bigr(\pr_{(\x,y)\sim \Dlabeled}[y\neq \concept(\x)] + \pr_{(\x,y)\sim \Dlabeled'}[y\neq \concept(\x)] \Bigr)
    \]
\end{definition}

Note that, in the absence of test labels, the dependence of the test error of the output hypothesis on $\lambda$ is unavoidable (see \cite{klivans2023testable}). We focus now on testable learning with distribution shift (TDS learning), where the learner is asked to either provide a hypothesis with near-optimal error on the test distribution, or detect the presence of harmful distribution shift.

\begin{definition}[TDS Learning \citep{klivans2023testable}]\label{definition:tds-learning}
    An algorithm $\A$ is a TDS-learner for concept class $\F$ with respect to some target distribution $\D^*$ over ${\R^d}$ if on input $(\epsilon,\delta,\bar S, S')$, where $\epsilon,\delta\in(0,1)$, $\bar S$ is a set of labeled i.i.d. examples from some training distribution $\Dlabeled$ where $\D = \D^*$ and $S'$ is a set of unlabeled i.i.d. examples from the marginal $\D'$ of some arbitrary test distribution $\bar \D'$, the algorithm $\A$ either outputs $\mathrm{Reject}$ or outputs $(\mathrm{Accept},h)$, where $h:{\R^d}\to \cube{}$ such that with probability at least $1-\delta$ over $\bar S, S$, and the randomness of $\A$, the following conditions hold. 
    \begin{enumerate}
        \item (\textit{Soundness}) Upon acceptance, the test error of $h$ is bounded as follows: 
        \[
            \pr_{(\x,y)\sim \Dlabeled'}[y \neq h(\x)] \le \lambda(\bar\D,\bar\D';\F)+ \min_{f\in\F} \pr_{(\x,y)\sim \Dlabeled}[y\neq \concept(\x)] +\eps
        \]
        \item (\textit{Completeness}) If $\D'=\Dtarget$, where $\D'$ is the marginal of the test distribution $\Dlabeled'$ on $\R^d$, then $\A$ accepts.
    \end{enumerate}
\end{definition}

Once more, we may combine \Cref{theorem:main} with a result from \cite{chandrasekaran2024efficient} to obtain the following theorem.

\begin{theorem}[Combination of \cite{klivans2023testable} with \Cref{theorem:main}]\label{theorem:tds-application}
    Let $\F$ be some concept class that satisfies \Cref{assumption:valid-concepts} 
    with parameters $k,d,\sigma$ with respect to a $\gamma$-strictly subexponential 
    distribution $\D^*$ over $\R^d$. There is a TDS-learner for $\F$ with respect to $\D^*$ that has runtime and sample complexity $(d\log 1/\delta)^{\widetilde{O} (\ell)}$ where we have:\footnote{The factor $\log(1/\delta)^{\widetilde{O}(\ell)}$ is inherited from \Cref{fact:polynomial-concentration}. Repetition-based amplification in \cite{klivans2023testable} incurs a constant error loss factor, so we instead obtain failure probability $\delta$ directly via \Cref{fact:polynomial-concentration}.}
    \[
        \ell
        \;\le\;
        \widetilde{O}\!\left(
        \frac{\sigma k^{3/2}}{\eps^{2}}
        \right)^{1+1/\gamma}.
    \]
\end{theorem}

In TDS learning, the algorithm accepts or rejects on a population basis. One natural question is whether it is possible for a learner to abstain on a per-point basis, meaning that it produces a hypothesis and a selector with the following guarantees.

\begin{definition}[PQ Learning \citep{goldwasser2020beyond}]\label{definition:pq-learning}
    An algorithm $\A$ is a PQ-learner for concept class $\F$ with respect to some target distribution $\D^*$ over ${\R^d}$ if on input $(\epsilon,\delta,\bar S, S')$, where $\epsilon,\eta,\delta\in(0,1)$, $\bar S$ is a set of labeled i.i.d. examples from some training distribution $\Dlabeled$ where $\D = \D^*$ and $S'$ is a set of unlabeled i.i.d. examples from the marginal $\D'$ of some arbitrary test distribution $\bar \D'$, the algorithm $\A$ either outputs a hypothesis $h:{\R^d}\to \cube{}$ and a selector $g:\R^d\to \{0,1\}$ such that with probability at least $1-\delta$ over $\bar S, S$, and the randomness of $\A$, the following conditions hold. 
    \begin{enumerate}
        \item (\textit{Error Rate}) The error of $h$ on the selected part of the test distribution is bounded as follows: 
        \[
            \pr_{(\x,y)\sim \Dlabeled'}[y \neq h(\x) \text{ and }g(\x) = 1] \le O\Biggr(\frac{\lambda(\bar\D,\bar\D';\F)}{\eta}\Biggr)+\eps
        \]
        \item (\textit{Rejection Rate}) The rejection rate of $g$ on the training distribution is bounded as follows:
        \[
            \pr_{\x\sim \D}[g(\x)=0] \le \eta
        \]
    \end{enumerate}
\end{definition}

The definition of PQ learning \citep{goldwasser2020beyond} predates that of TDS learning \citep{klivans2023testable}, although PQ learning constitutes a strictly more demanding learning primitive. Indeed, PQ learning implies a tolerant version of TDS learning \citep{klivans2023testable,goel2024tolerant}. However, the first end-to-end efficient algorithms for PQ learning \citep{goel2024tolerant} were obtained by building on techniques developed for TDS learning \citep{klivans2023testable}, through the use of sandwiching polynomials. Prior to these results, \cite{goldwasser2020beyond,kalai2021efficient} established reductions between PQ learning and more classical learning primitives, such as distribution-free reliable agnostic learning. Such primitives, however, are likely computationally intractable even for very simple concept classes, including conjunctions (see \cite{kalai2021efficient}).

We provide the following positive result for PQ learning, which uses---once more---polynomial hypercontractivity, and the result of \cite{goel2024tolerant}. Note that the result of \cite{goel2024tolerant} requires the existance of $\L_2$-sandwiching polynomials, instead of standard $\L_1$-sandwiching. However, our techniques yield results for sandwiching of any order $s\ge 1$.

\begin{theorem}[Combination of \cite{goel2024tolerant} with \Cref{theorem:main}]\label{theorem:pq-application}
    Let $\F$ be some concept class that satisfies \Cref{assumption:valid-concepts} 
    with parameters $k,d,\sigma$ with respect to a $\gamma$-strictly subexponential and polynomially hypercontractive 
    distribution $\D^*$ over $\R^d$. There is a PQ-learner for $\F$ with respect to $\D^*$ that has runtime and sample complexity $(d\log 1/\delta)^{\widetilde{O}(\ell)}$ where we have:
    \[
        \ell
        \;\le\;
        \widetilde{O}\!\left(
        \frac{\sigma k^{3/2}}{\eps^{3/2}}
        \right)^{1+1/\gamma}.
    \]
\end{theorem}

Note that the dependence on $1/\eps$ here improves upon that in \Cref{theorem:tds-application}. This improvement arises because, in the PQ learning setting, we allow the learner to return a constant-factor approximate solution, namely an error of $O(\lambda/\eta)$ (which reduces to $O(\lambda)$ when $\eta \approx 1$), whereas our definition of TDS learning requires a strictly tighter error guarantee. In particular, our analysis yields an upper bound on the test error that is proportional to $\|\pup-\pdown\|_{\D,2}^2$, rather than $\|\pup-\pdown\|_{\D,2}$. This quadratic dependence is one of the reasons a constant multiplicative error factor appears, but it allows us to work with $(\sqrt{\eps},2)$-sandwiching polynomials.

It remains unclear whether tighter error guarantees are achievable for PQ learning. In particular, it is not known whether $\mathcal L_1$-sandwiching alone suffices, while $\mathcal L_2$ approximation is known to incur a constant-factor loss in agnostic learning \citep{kalai2008agnostically}.

\paragraph{Learning with Heavy Contamination} Our final learning-theoretic application is learning with heavy contamination. This primitive was recently defined by \cite{klivans2025the}. In the heavy contamination, the input consists mostly of examples that are adversarial, with a constant fraction of clean points.

\begin{definition}[Heavily Contaminated (HC) Datasets \cite{klivans2025the}]\label{definition:heavy-contamination}
    \textit{Let $\Dlabeled$ be some distribution over $\R^d\times \cube{}$ (we think of $\Dlabeled$ as the clean or uncorrupted distribution). We say that a set of samples $\Sinplabeled$ is generated by $\Dlabeled$ with $Q$-heavy contamination if it is generated as follows for some $m\le M$ with $M/m \le Q$. 
    \begin{enumerate}
        \item First, a set $\Sclnlabeled$ of $m$ i.i.d. labeled examples from $\Dlabeled$ is drawn.
        \item Then, an adversary receives $\Sclnlabeled$ and adds $M-m$ arbitrary labeled examples to form $\Sinplabeled$.
    \end{enumerate}}
\end{definition}

The error benchmark in the HC setting is defined as follows, and in particular requires that there is a classifier in the considered concept class that classifies almost all of the input examples correctly.

\begin{definition}[Error Benchmark in the Heavy Contamination Setting]\label{definition:opttotal}
    Let $\Sinplabeled$ be a set of labeled examples in $\R^d\times\cube{}$ and $\F$ some concept class over $\R^d$. We define the following error benchmark:
    \[
    \opttotal(\Sinplabeled;\F) = \min_{f\in\F}\frac{1}{|\Sinplabeled|} \sum_{(\x,y)\in \Sinplabeled}\ind\{y\neq f(\x)\}
    \]
\end{definition}

Note that empirical risk minimization over VC classes achieves an error guarantee of $Q\cdot \opttotal+\eps$ in the HC setting. The more interesting question is whether there are any efficient algorithms that can achieve the same guarantee, given that the input samples are not structured, apart from a small fraction. As we have said, structure is important for efficient learnability, but the question is whether HC datasets are structured enough. We now provide the definition of HC-learning.

\begin{definition}[HC-Learning \citep{klivans2025the}]\label{definition:hc-learning}    
    An algorithm $\A$ is an HC-learner for concept class $\F$ with respect to some target distribution $\D^*$ over $\R^d$ if on input $(\epsilon,\delta,Q,\Sinplabeled)$, where $\epsilon,\delta\in(0,1)$, $Q\ge 1$ and $\Sinplabeled$ is a $Q$-heavily contaminated set of labeled examples generated by distribution $\Dlabeled$ whose marginal on $\R^d$ is $\D = \Dtarget$ (as described in Definition \ref{definition:heavy-contamination}), the algorithm $\A$ outputs some hypothesis $h:\R^d\to \cube{}$ such that with probability at least $1-\delta$ over the clean examples in $\Sinplabeled$, and the randomness of $\A$: 
    \[
        \pr_{(\x,y)\sim \Dlabeled}[y\neq h(\x)] \le Q\cdot \opttotal(\Sinplabeled;\F) + \epsilon 
    \]
\end{definition}

\cite{klivans2025the} showed that the existence of low-degree $\L_1$-sandwiching polynomials suffices for efficient HC-learning when the clean distribution is hypercontractive. We obtain the following theorem by combining their result with \Cref{theorem:main}.

\begin{theorem}[Combination of \cite{klivans2025the} with \Cref{theorem:main}]\label{theorem:hc-application}
    Let $\F$ be some concept class that satisfies \Cref{assumption:valid-concepts} 
    with parameters $k,d,\sigma$ with respect to a $\gamma$-strictly subexponential and polynomially hypercontractive 
    distribution $\D^*$ over $\R^d$. There is an HC-learner for $\F$ with respect to $\D^*$ that has runtime and sample complexity $(d\log 1/\delta)^{\widetilde{O}(\ell)}$ where we have:
    \[
        \ell
        \;\le\;
        \widetilde{O}\!\left(
        \frac{\sigma k^{3/2}}{\eps^{2}}
        \right)^{1+1/\gamma}.
    \]
\end{theorem}

\subsection{Pseudorandomness}\label{section:pseudorandomness-application}

Our work also has implications for pseudorandomness, where the goal is to replace a distribution such as the Gaussian with another distribution that can be generated using a small number $r$ of random bits while preserving certain test statistics. This is central to derandomization, since when $r$ is small, one can efficiently enumerate over all $2^r$ possible seeds and thereby simulate randomized algorithms deterministically with only a modest overhead (see \cite{hatami2023theory} and references therein).

A central paradigm in pseudorandomness is to fool a class of test functions using distributions with limited randomness, such as bounded independence or, more generally, moment matching. While more sophisticated techniques usually lead to pseudorandom generators with shorter seed lengths, moment matching leads to simple, explicit pseudorandom generators and has been thoroughly studied in the relevant literature \citep{bazzi2009polylogarithmic,razborov2009simple,gopalan2010fooling,diakonikolas2010bounded,braverman2011poly,kane2011kindependent,diakonikolas2010ptf,tal2017tight,hatami2023theory}.

\begin{definition}[Moment Matching Distributions]\label{definition:moment-matching}
    Let $\D$ be a distribution over $\R^d$. For $\ell\in\nats, \Delta>0$, we define $\mm_{\D}(\ell,\Delta)$ to be the set of distributions $\D'$ such that for any $\alpha\in\nats^d$ with $\|\alpha\|_0 \le \ell$ we have:
    \[
        \Bigr|\E_{\x\sim \D}[\x^\alpha] - \E_{\x\sim \D'}[\x^\alpha]\Bigr| \le \Delta
    \]
\end{definition}

Having the definition of moment matching at hand, we are ready to give a definition for fooling via moment matching.

\begin{definition}[Fooling via Moment Matching]\label{definition:fooling}
    Let $\F$ be a concept class over $\R^d$ and $\D$ a distribution over $\R^d$. We say that $(\ell,\Delta)$-moment matching $\eps$-fools $\F$ with respect to $\D$ if the following is true for any distribution $\D'\in\mm_{\D}(\ell,\Delta)$:
    \[
        \Bigr| \E_{\x\sim \D}[f(\x)] - \E_{\x\sim \D'}[f(\x)] \Bigr| \le \eps\,,\quad \text{ for all } f\in\F
    \]
\end{definition}

Our result is the following fooling result for geometric concepts with low instrinsic dimension.

\begin{theorem}[Fooling Geometric Concepts with Low Intrinsic Dimension]\label{theorem:fooling-result}
        Let $\F$ be some concept class that satisfies \Cref{assumption:valid-concepts} 
        with parameters $k,d,\sigma$ with respect to a $\gamma$-strictly subexponential 
        distribution $\D$ over $\R^d$. Then, for any $\eps\in(0,1)$ and $\Delta\ge 0$, we have that $(\ell,\Delta)$-moment matching $\eps'$-fools $\F$ with respect to $\D$, where $\eps' = \eps + \Delta B$ and
    \[
        \ell
        \;\le\;
        \widetilde{O}\!\left(
        \frac{\sigma k^{3/2}}{\eps^2}
        \right)^{1+1/\gamma}\,, \quad B \le \exp\Biggr(\widetilde{O}\left(
        \frac{\sigma k^{3/2}}{\eps^2}
        \right)^{1+1/\gamma}\Biggr)
    \]
\end{theorem}

The result of the above theorem is a corollary of \Cref{theorem:main}, combined with the following theorem which was first proven by \cite{bazzi2009polylogarithmic} for the case $\Delta=0$, and then for $\Delta>0$ by \cite{gollakota2022moment}. The result states that the sandwiching degree characterizes the degree of fooling via moment matching, and the proof is based on linear programming duality.

\begin{theorem}[\cite{bazzi2009polylogarithmic,gollakota2022moment}]\label{theorem:pseudorandomness-via-moment-matching}
    The following hold for any $\eps,\ell,\Delta,B > 0$ and $f:\R^d\to\cube{}$:
    \begin{itemize}
        \item If there exist $(\eps,1)$-sandwiching polynomials for $f$ with respect to $\D$ whose degree is at most $\ell$ and and the absolute values of their coefficients sum to $B$, then $(\ell,\Delta)$-moment matching $\eps'$-fools $f$ with respect to $\D$, where $\eps' = \eps + \Delta B$.
        \item Conversely, if $(\ell,\Delta)$-moment matching $\eps$-fools $f$ with respect to $\D$, then there are $(2\eps,1)$-sandwiching polynomials for $f$ with respect to $\D$ whose degree is at most $\ell$ and the sum of the absolute values of their coefficients is at most $B' = 2\eps/\Delta$.\footnote{In the case $\Delta=0$, there is no bound for the coefficients of the sandwiching polynomials in general.}
    \end{itemize}
\end{theorem}

\bibliographystyle{alpha}
\bibliography{refs}

\newpage
\appendix

\section{Gaussian Surface Area and Rectifiable Boundaries}\label{section:GSA-rectifiable}

In \Cref{section:sandwiching-bounds-for-classes}, we provide sandwiching degree bounds for several fundamental concept classes. To this end, it is important to obtain sharp bounds on the boundary smoothness parameter. Prior work in geometric measure theory as well as learning theory (see, e.g., \cite{nazarov2003maximal,klivans2008learning}) has focused on the notion of Gaussian surface area, which is essentially an (one-sided) asymptotic version of the boundary smoothness parameter defined as follows.

\begin{definition}[Gaussian Surface Area]
    Let $f:\R^d\to \cube{}$. The Gaussian Surface Area of $f$ is defined as follows:
    \[
        \gsa(f) := \lim_{\rho \to 0^+} \frac{\E_{\x\sim\Gauss_d}[f^{+\rho}(\x)-f(\x)]}{2\rho}
    \]
\end{definition}

In particular, the following result by Nazarov gives a sharp bound on the Gaussian surface area of halfspace intersections.

\begin{fact}[Nazarov, see \cite{klivans2008learning}]
    Let $f:\R^d\to \cube{}$ be an intersection of $k$ halfspaces. Then: 
    \[\gsa(f)\le \sqrt{2\ln k} + 2\]
\end{fact}

In order to be able to transform bounds on the surface area to bounds on boundary smoothness, we require the following regularity condition from geometric measure theory.

\begin{definition}[Rectifiable Set \citep{federer1969geometric}]\label{definition:rectifiable}
    Let $K\subseteq \R^d$. We say that $K$ is $(d-1)$-rectifiable if there exists a countable collection of Lipschitz continuous maps $\{G_i:\R^{d-1}\to \R^d\}_{i}$ such that:
    \[
        \H^{d-1}\Bigr(K \setminus \bigcup_{i=1}^{\infty}G_i(\R^{d-1}) \Bigr) = 0\,,
    \]
    where $\H^{d-1}$ is the $(d-1)$-dimensional Hausdorff measure.
\end{definition}

The above condition defines rectifiable sets as the sets in $\R^d$ that effectively behave like countable unions of surfaces defined by Lipschitz continuous transformations from $\R^{d-1}$ to $\R^d$.

We will use a result from geometric measure theory that equates the Gaussian surface area of a concept to the Gaussian-weighted surface integral of its decision boundary, whenever the decision boundary is rectifiable. We provide the formal definition of the decision boundary, as well as the aforementioned result below.

\begin{definition}[Boundary of a Concept]\label{definition:boundary}
    Let $f:\R^d\to \cube{}$. We define the boundary $\partial f $ of $f$ to be the set of points $\x$ where for any $\rho>0$ there exists $\z\in\R^d$ with $\|\z\|_2\le \rho$ and $f(\x)\neq f(\x+\z)$.
\end{definition}

\begin{fact}[See \cite{federer1969geometric}]
    Let $f:\R^d\to \cube{}$ such that $\partial f$ is $(d-1)$-rectifiable. Then we have:
    \[
    \gsa(f) = \int_{\partial f} \phi(\x)\,d\H^{d-1}(\x)\,,\]
    where $\phi$ is the standard Gaussian density $\phi(\x) = (2\pi)^{-d/2} \exp(-\|\x\|_2^2/2)$ and $\H^{d-1}$ is the $(d-1)$-dimensional Hausdorff measure. Moreover, if $f$ indicates a convex set then $\partial f$ is $(d-1)$-rectifiable.
\end{fact}

\section{Polynomial Concentration under Subexponential Distributions}\label{section:concentration}

Throughout this work, we have focused on distributions that are strictly subexponential. For our applications to learning theory, we use the following fact regarding the uniform convergence for polynomials with bounded coefficients, with respect to subexponential distributions.

\begin{fact}\label{fact:polynomial-concentration}
    Let $\P(\ell,B)$ be the family of polynomials $p$ of degree at most $\ell$ and $\|p\|_\coef \le B$. Let $\D$ be some distribution over $\R^d$ that is subexponential and let $\bar\D$ be any labeled distribution over $\R^d\times\cube{}$ whose $\R^d$-marginal is $\D$. Then, the following holds with probability at least $1-\delta$ over a set $\bar S$ of $m \ge (B/\eps)^{O(1)} (d\cdot \log 1/\delta)^{\widetilde{O}(\ell)}$ samples drawn independently from $\bar \D$:
    \[
        \biggr| \E_{(\x,y)\sim\bar\D}\Bigr[ (y-p(\x))^2 \Bigr] - \E_{(\x,y)\sim \bar S} \Bigr[ (y-p(\x))^2 \Bigr] \biggr| \le \eps \,,\quad \text{for all }p\in\P
    \]
\end{fact}

The proof of \Cref{fact:polynomial-concentration} is based on standard thresholding arguments (e.g., see the proof of Lemma 59 in \cite{klivans2023testable}) combined with classical generalization results for kernelized regression (see, e.g., \cite{mohri2018foundations}), as well as the following result on the concentration of monomials under subexponential distributions.

\begin{fact}\label{fact:monomial-concentration}
    Let $\D$ be some distribution over $\R^d$ that is subexponential. Then, the following holds with probability at least $1-\delta$ over a set $S$ of $m \ge (1/\eps)^{2} (\log 1/\delta)^{\widetilde{O}(\ell)}$ independent samples from $\D$:
    \[ \biggr| \E_{(\x,y)\sim\D}\Bigr[ \x^\I \Bigr] - \E_{(\x,y)\sim  S} \Bigr[ \x^\I\Bigr] \biggr| \le \eps \,,\quad \text{for all }\I\in\N^d: \|\I\|_1 \le \ell \]
\end{fact}

The proof of \Cref{fact:monomial-concentration} is an application of the following result from probability theory.

\begin{theorem}[Marcinkiewicz–Zygmund, see \cite{FERGER201496marcinkiewicz}]\label{theorem:marcinkiewicz}
    Let $\D$ be a distribution over the reals, let $q\ge 2$, and let $S$ be a set of $m$ i.i.d. examples from $\D$. We have the following:
    \[
        \E_{S\sim \D^{\otimes m}}\biggr[ \biggr| \frac{1}{m}\sum_{x\in S}\Bigr(x - \E_{x'\sim\D}[x']\Bigr) \biggr|^q \biggr] \le \frac{2\, q^{q/2}}{m^{q/2}} \E_{x\sim \D}\biggr[ \biggr| x - \E_{x'\sim\D}[x']\biggr|^{q} \biggr]
    \]
\end{theorem}

In particular, to obtain \Cref{fact:monomial-concentration}, we may use Markov's inequality on the random variable $|x_i^{\ell}-\E[x_i^\ell]|^q$, where $\x$ follows some subexponential distribution, and then apply \Cref{theorem:marcinkiewicz} on the random variable $x = x_i^{\ell}$.

\end{document}